\documentclass[10pt]{article}

\usepackage{fullpage}                   %
\usepackage{authblk}

\usepackage{amsmath, amssymb, amsthm}  %
\usepackage{graphicx}                   %
\usepackage{enumitem}                   %
\usepackage{bm}                         %

\usepackage[T1]{fontenc}
\usepackage{booktabs}
\usepackage[misc]{ifsym}
\usepackage[table]{xcolor}
\usepackage[nosort]{cite}

\usepackage{xcolor}
\usepackage{algorithm}
\usepackage{algorithmic}
\usepackage{booktabs}
\usepackage[hidelinks]{hyperref}
\hypersetup{
    colorlinks=true, %
    linkcolor=darkblue, %
    anchorcolor=darkblue, %
    citecolor=darkblue, %
    filecolor=darkblue, %
    menucolor=darkblue, %
    urlcolor=darkblue %
}

\usepackage[nameinlink,noabbrev]{cleveref}
\usepackage{graphicx}
\usepackage{caption}
\usepackage{subcaption}
\usepackage{adjustbox}
\usepackage[pass]{geometry}
\usepackage{multirow}

\newcommand{\norm}[1]{\|#1\|}

\newcommand{\nsq}[1]{\|#1\|^2}
\newcommand{\bnsq}[1]{\big\|#1 \big\|^2}

\newcommand{\proj}{\mathrm{proj}}
\newcommand{\dist}{\mathrm{dist}}
\newcommand{\ip}[2]{\langle#1,#2\rangle}
\newcommand{\avgn}{\frac{1}{n} \sum_{i=1}^{n}}
\newcommand{\R}{\mathbb{R}}
\newcommand{\argmin}[1]{\underset{#1}{\mathrm{argmin}}}
\newcommand{\x}{\mathbf{x}}

\newcommand{\xb}{\bar{\mathbf{x}}}
\newcommand{\g}{\mathbf{g}}

\newcommand{\s}{\mathbf{s}}

\newcommand{\dom}{\mathcal{D}}
\newcommand{\y}{\mathbf{y}}
\newcommand{\bd}{\mathbf{d}}

\newcommand{\X}{\mathbf{X}}
\newcommand{\Y}{\mathbf{Y}}
\newcommand{\D}{\mathbf{D}}

\newcommand{\Xt}{\mathbf{X}^t}

\newcommand{\Xtt}{\mathbf{X}^{\tilde{t}}}
\newcommand{\Xb}{\bar{\mathbf{X}}}
\newcommand{\Xbt}{\bar{\mathbf{X}}^t}
\newcommand{\Xbtt}{\bar{\mathbf{X}}^{\tilde{t}}}
\newcommand{\Lhat}{\hat{L}_{t}}

\newcommand{\C}{\mathcal{C}}
\newcommand{\gap}{\mathrm{gap}}
\newcommand{\sgap}{\hat{\mathrm{gap}}}
\definecolor{alizarin}{rgb}{0.82, 0.1, 0.26}
\definecolor{amber}{rgb}{1.0, 0.49, 0.0}
\definecolor{darkblue}{rgb}{0.0, 0.0, 0.55} 
\definecolor{ao}{rgb}{0.0, 0.0, 1.0}
\definecolor{ao(english)}{rgb}{0.0, 0.5, 0.0}
\definecolor{bostonuniversityred}{rgb}{0.8, 0.0, 0.0}
\newcommand{\mycolorOne}[1]{{\color{ao}{#1}}}
\newcommand{\mycolorTwo}[1]{{\color{bostonuniversityred}{#1}}}

\newtheorem{theorem}{Theorem}
\newtheorem{lemma}{Lemma}

\theoremstyle{definition}

\newtheorem{remark}{Remark}

\crefname{assumption}{assumption}{assumptions}
\Crefname{assumption}{Assumption}{Assumptions}
\newtheorem{assumption}{Assumption}

\crefname{definition}{definition}{definitions}
\Crefname{definition}{Definition}{Definitions}
\title{Federated Frank-Wolfe Algorithm}
\author[1]{Ali Dadras}
\author[2]{Sourasekhar Banerjee}
\author[1]{Karthik Prakhya}
\author[1]{Alp Yurtsever}
\affil[1]{Department of Mathematics and Mathematical Statistics}
\affil[2]{Department of Computing Science}
\affil[ ]{Umeå University, Sweden}
\affil[ ]{\texttt{\{name.surname\}@umu.se}}
\date{}

\begin{document}

\maketitle

\begin{abstract}
Federated learning (FL) has gained a lot of attention in recent years for building privacy-preserving collaborative learning systems. 
However, FL algorithms for constrained machine learning problems are still limited, particularly when the projection step is costly. 
To this end, we propose a Federated Frank-Wolfe Algorithm (\textsc{FedFW}). 
\textsc{FedFW} features data privacy, low per-iteration cost, and communication of sparse signals. 
In the deterministic setting, \textsc{FedFW} achieves an $\varepsilon$-suboptimal solution within $\mathcal{O}(\varepsilon^{-2})$ iterations for smooth and convex objectives, and $\mathcal{O}(\varepsilon^{-3})$ iterations for smooth but non-convex objectives. 
Furthermore, we present a stochastic variant of \textsc{FedFW} and show that it finds a solution within $\mathcal{O}(\varepsilon^{-3})$ iterations in the convex setting. 
We demonstrate the empirical performance of \textsc{FedFW} on several machine learning tasks. 
\end{abstract}


\section{Introduction}
\label{sec:Introduction}

We present a new variant of the Frank-Wolfe (FW) algorithm, \textsc{FedFW}, designed for the increasingly popular Federated Learning (FL) paradigm in machine learning. 
Consider the following constrained empirical risk minimization  template:
\begin{equation}\label{eqn:sec:intro:problem def}
    \min_{\x \in \dom} ~~ F(\x):= ~ \frac{1}{n} \sum_{i=1}^n f_i(\x),
\end{equation}
where $\dom \subseteq \R^p$ is a convex and compact set. We define the diameter of $\dom$ as $D:=\max_{\x, \mathbf{y} \in \dom} \|\x-\mathbf{y}\|$. 
The function $F : \R^p \to \R$ represents the objective function, and $f_i : \R^p \to \R$ (for $i = 1,\ldots,n$) represent the loss functions of the clients, where $n$ is the number of clients. 
Throughout, we assume $f_i$ is $L$-smooth, meaning that it has Lipschitz continuous gradients with parameter $L$. 

FL holds great promise for solving optimization problems over a large network, where clients collaborate under the coordination of a server to find a common good model. 
Privacy is an explicit goal in FL; clients work together towards a common goal by utilizing their own data without sharing it. 
As a result, FL exhibits remarkable potential for data science applications involving privacy-sensitive information. 
Its applications range from learning tasks (such as training neural networks) on mobile devices without sharing personal data \cite{lim2020federated} to medical applications of machine learning, where hospitals collaborate without sharing sensitive patient information \cite{wang2021field}. 

Most FL algorithms focus on unconstrained optimization problems, and extending these algorithms to handle constrained problems typically requires projection steps. 
However, in many machine learning applications, the projection cost can create a computational bottleneck, preventing us from solving these problems at a large scale. 
The FW algorithm \cite{frank1956algorithm} has emerged as a preferred method for addressing these problems in machine learning. 
The main workhorse of the FW algorithm is the linear minimization oracle (LMO),
\begin{equation}
    \mathrm{lmo}(\mathbf{y}) := \argmin{\x \in\dom} ~ \ip{\mathbf{y}}{\x}.
\end{equation}
Evaluating linear minimization is generally less computationally expensive than performing the projection step. 
A famous example illustrating this is the nuclear-norm constraint: projecting onto a nuclear-norm ball often requires computing a full-spectrum singular value decomposition. 
In contrast, linear minimization involves finding the top singular vector, a task that can be efficiently approximated using methods such as the power method or Lanczos iterations.

To our knowledge, FW has not yet been explored in the context of FL. This paper aims to close this gap. 
Our primary contribution lies in adapting the FW method for FL with convergence guarantees. 

The paper is organized as follows:
\Cref{section: Related Work} provides a brief review of the literature on FL and the FW method. 
In \Cref{section: Algorithms and Convergence Analysis}, we introduce \textsc{FedFW}. Unlike traditional FL methods, \textsc{FedFW} does not overwrite clients' local models with the global model sent by the server. 
Instead, it penalizes clients' loss functions by using the global model. We present the convergence guarantees of \textsc{FedFW} in \Cref{section:guarantees}. 
Specifically, our method provably finds a $\varepsilon$-suboptimal solution after $\mathcal{O}(\varepsilon^{-2})$ iterations for smooth and convex objective functions (refer to \Cref{theorem:sec:algo. and convergence anal.: convex}). 
In the case of non-convex objectives, the complexity increases to $\mathcal{O}(\varepsilon^{-3})$ (refer to \Cref{theorem:sec:algo. and convergence anal.: non-convex}). 
\Cref{sec:design-variants} introduces several design variations of \textsc{FedFW}, including a stochastic variant. 
\Cref{section: Numerical Experiments} presents numerical experiments on various machine learning tasks with both convex and non-convex objective functions. 
Finally, \Cref{section: Conclusions} provides concluding remarks along with a discussion on the limitations of the proposed method. 
Detailed proofs and technical aspects are deferred to the supplementary material.

\section{Related Work} 
\label{section: Related Work}

\paragraph{Federated Learning.}
FL is a distributed learning paradigm that, unlike most traditional distributed settings, focuses on a scenario where only a subset of clients participate in each training round, data is often heterogeneous, and clients can perform different numbers of iterations in each round \cite{mcmahan2017communication,konevcny2016federated}. 
\textsc{FedAvg} \cite{mcmahan2017communication} has been a cornerstone in the FL literature, demonstrating practical capabilities in addressing key concerns such as privacy and security, data heterogeneity, and computational costs. 
Although it is shown that fixed points of some \textsc{FedAvg} variants do not necessarily converge to the minimizer of the objective function, even in the least squares problem \cite{pathak2020fedsplit}, and can even diverge \cite{zhang2021fedpd}, the convergence guarantees of \textsc{FedAvg} have been studied under different assumptions (see  \cite{stich2018local,li2019convergence,haddadpour2019local,yu2019parallel,li2020federated,woodworth2020local,woodworth2020minibatch,al2020federated} and the references therein). 
However, all these works on the convergence guarantees of \textsc{FedAvg} are restricted to unconstrained problems. 

Constrained or composite optimization problems are ubiquitous in machine learning, often used to impose structural priors such as sparsity or low-rankness. 
To our knowledge, \textsc{FedDR} \cite{tran2021feddr} and \textsc{FedDA} \cite{yuan2021federated} are the first FL algorithms with convergence guarantees for constrained problems. 
The former employs Douglas-Rachford splitting, while the latter is based on the dual averaging method \cite{nesterov2009primal}, to solve composite optimization problems, including constrained problems via indicator functions, within the FL setting. 
\cite{bao2022fast} introduced a `fast' variant of \textsc{FedDA}, achieving rapid convergence rates with linear speedup and reduced communication rounds for composite strongly convex problems. 
\textsc{FedADMM} \cite{wang2022fedadmm} was proposed for federated composite optimization problems involving a non-convex smooth term and a convex non-smooth term in the objective. 
Moreover, \cite{he2023federated} proposed a FL algorithm based on a proximal augmented Lagrangian approach to address problems with convex functional constraints. 
None of these works address our problem template, where the constraints are challenging to project onto but allow for an efficient solution to the linear minimization problem. 

\paragraph{Frank-Wolfe Algorithm.}
The FW algorithm, also known as the conditional gradient method or CGM, was initially introduced in \cite{frank1956algorithm} to minimize a convex quadratic objective over a polytope, and was extended to general convex objectives and arbitrary convex and compact sets in \cite{levitin1966constrained}. Following the seminal works in \cite{hazan2008sparse} and \cite{jaggi2013revisiting}, the method gained popularity in machine learning. 

The increasing interest in FW methods for data science applications has led to the development of new results and variants. 
For example, \cite{lacoste2016convergence} established convergence guarantees for FW with non-convex objective functions. 
Additionally, online, stochastic, and variance-reduced variants of FW have been proposed; see \cite{hazan2012projectionfree,hazan2016variance,reddi2016stochastic,yurtsever2019spider,mokhtari2020stochastic,negiar2020stochastic} and the references therein. 
FW has also been combined with smoothing strategies for non-smooth and composite objectives \cite{lan2012optimal,yurtsever2018conditional,locatello2019stochastic,dresdner2022faster}, and with augmented Lagrangian methods for problems with affine equality constraints \cite{gidel2018frank,yurtsever2019conditional}. 
Furthermore, various design variants of FW, such as the away-step and pairwise step strategies, can offer computational advantages. 
For a comprehensive overview of FW-type methods and their applications, we refer to \cite{kerdreux2020accelerating,bomze2021frank}.

The most closely related methods to our work are the distributed FW variants. However, the variants in 
\cite{wai2017decentralized,mokhtari2018decentralized,gao2021sample} are fundamentally different from \textsc{FedFW} as they require sharing gradient information of the clients with the server or with the neighboring nodes. In \textsc{FedFW}, clients do not share gradients, which is critical for data privacy \cite{zhu2019deep,li2022auditing}. Other distributed FW variants are proposed in \cite{wang2016parallel,zheng2018distributed,zhang2021decentralized}. However, the method proposed by \cite{zheng2018distributed} is limited to the convex low-rank matrix optimization problem, and the methods in \cite{wang2016parallel} and \cite{zhang2021decentralized} assume that the problem domain is block separable.

\section{Federated Frank-Wolfe Algorithm}
\label{section: Algorithms and Convergence Analysis}

In essence, any first-order optimization algorithm can be adapted for a simplified federated setting by transmitting local gradients to the server at each iteration. 
These local gradients can be aggregated to compute the full gradient and distributed back to the clients. 
Although it is possible to implement the standard FW algorithm in FL this way, this baseline has two major problems. 
First, it relies on communication at each iteration, which raises scalability concerns, as extending this approach to multiple local steps is not feasible. 
Secondly, sharing raw gradients raises privacy concerns, as sensitive information and data points can be inferred with high precision from transmitted gradients \cite{zhu2019deep}. 
Consequently, most FL algorithms are designed to exchange local models or step-directions rather than gradients. 
Unfortunately, a simple combination of the FW algorithm with a model aggregation step fails to find a solution to \eqref{eqn:sec:intro:problem def}, as we demonstrate with a simple counterexample in the supplementary material. 
Therefore, developing \textsc{FedFW} requires a special algorithmic approach, which we elaborate on below.

We start by rewriting problem~\eqref{eqn:sec:intro:problem def} in terms of the matrix decision variable $\mathbf{X} := [\x_1,\x_2,\ldots,\x_n]$, as follows:
\begin{equation}\label{eqn:sec:algo. and convergence anal.:composite equivalent}
    \min_{\mathbf{X} \in \mathcal{D}^n} ~ \frac{1}{n} \sum_{i=1}^n f_i(\mathbf{X} \mathbf{e}_i) + \delta_\C(\mathbf{X}).
\end{equation}
Here, $\mathbf{e}_i$ denotes the $i$th standard unit vector, and $\delta_\C$ is the indicator function for the consensus set:
\begin{equation}
\mathcal{C} := \{[\x_1,\ldots,\x_n] \in \mathbb{R}^{p\times n}: \x_1 = \x_2 = \ldots = \x_n \}.
\end{equation}
It is evident that problems \eqref{eqn:sec:intro:problem def} and \eqref{eqn:sec:algo. and convergence anal.:composite equivalent} are equivalent. 
However, the latter formulation represents the local models of the clients as the columns of the matrix $\mathbf{X}$, offering a more explicit representation for FL. 

The original FW algorithm is ill-suited for solving problem \eqref{eqn:sec:algo. and convergence anal.:composite equivalent} due to the non-smooth nature of the objective function because of the indicator function. Drawing inspiration from techniques proposed in \cite{yurtsever2018conditional}, we adopt a quadratic penalty strategy to address this challenge. The main idea is to perform FW updates on a surrogate objective which replaces the hard constraint $\delta_{\mathcal{C}}$ with a smooth function that penalizes the distance between $\mathbf{X}$ and the consensus set $\mathcal{C}$: 
\begin{equation}\label{eqn: surrogate function}
    \hat{F}_t(\mathbf{X}) = \frac{1}{n} \sum_{i=1}^n f_i(\mathbf{X} \mathbf{e}_i) + \frac{\lambda_t}{2} \dist^2(\mathbf{X},\C), 
\end{equation}
where $\lambda_t \geq 0$ is the penalty parameter. 
Note that the surrogate function is parameterized by the iteration counter $t$, as it is crucial to amplify the impact of the penalty function by gradually increasing $\lambda_t$ at a specific rate through the iterations. 
This adjustment will ensure that the generated sequence converges to a solution of the original problem in \eqref{eqn:sec:algo. and convergence anal.:composite equivalent}.

To perform an FW update with respect to the surrogate function, first, we need to compute the gradient of $\hat{F}_t$, given by
\begin{equation}\label{eqn:sec:algo. and convergence anal.:gradient of composite equivalent}
\begin{aligned}
    \nabla \hat{F}_t(\mathbf{X}) 
    & = \frac{1}{n} \sum_{i=1}^n \nabla f_i(\mathbf{X} \mathbf{e}_i)  \mathbf{e}_i^\top + \lambda_t (\mathbf{X} - \proj_\C(\mathbf{X})) \\
    & = \frac{1}{n} \sum_{i=1}^n \nabla f_i(\x_i)  \mathbf{e}_i^\top + \lambda_t \sum_{i=1}^n (\x_i - \xb)  \mathbf{e}_i^\top 
\end{aligned}
\end{equation}
where $\xb := \frac{1}{n} \sum_{i=1}^n \x_i$.
Then, we call the linear minimization oracle:
\begin{equation}\label{eqn:lmo-combined}
    \mathbf{S}^t \in \argmin{\mathbf{X} \in \dom^n} ~ \ip{\nabla \hat{F}_t(\mathbf{X}^t)}{\mathbf{X}}.
\end{equation}
Since $\mathcal{D}^n$ is separable for the columns of $\mathbf{X}$, we can evaluate \eqref{eqn:lmo-combined} in parallel for $\x_1,\x_2,\ldots,\x_n$. Define $\s_i^t$ as
\begin{equation}\label{eqn:definition of step direction}
 \s_i^t \in \argmin{\x \in \dom} ~ \ip{\frac{1}{n}\nabla f_i(\x_i^t) + \lambda_t (\x_i^t - \xb^t)}{\x},
\end{equation}
where $ \xb^t := \frac{1}{n} \sum_{i=1}^n \x_i^t$. Then, 
$\mathbf{S}^t = \sum_{i=1}^n \s_i^t \, \mathbf{e}_i^\top.$

Finally, we update the decision variable by $\mathbf{X}^{t+1} = (1-\eta_t) \mathbf{X}^t + \eta_t \mathbf{S}^t$, which can be computed column-wise in parallel:
\begin{equation}
    \x_i^{t+1} = (1-\eta_t) \x_i^t + \eta_t \mathbf{s}_i^t,
\end{equation}
where $\eta_t \in [0,1]$ is the step-size. 

This establishes the fundamental update rule for our proposed algorithm, \textsc{FedFW}. 
Note that communication is required only during the computation of $\xb^t$, which constitutes our aggregation step. 
All other computations can be performed locally by the clients. 
\Cref{algo_all} presents \textsc{FedFW} and several design variants, which are further detailed in \Cref{sec:design-variants}.

\begin{algorithm}[t]
\caption{\textsc{FedFW}: Federated Frank-Wolfe Algorithm ($+$variants)}
\label{algo_all} 
\begin{algorithmic}
\STATE \textbf{input} $\ \x_i^{1} \in \mathbb{R}^p,  ~ \forall i\in[n], ~ \lambda_t, ~\eta_t, ~\mycolorTwo{\rho_t}, ~ \xb^{1} = \frac{1}{n} \sum_{i=1}^n \x_i^1, ~ \mycolorOne{\y_i^1 = 0}, ~ \mycolorTwo{\bd_i^1 = 0}$
\FOR {round $t = 1,2,\ldots,T$}
        \vspace{0.25em}
        \STATE --- \textbf{Client}-level local training ----------------------------------- 
        \FOR {client $i = 1,2,\ldots,n$}
        \vspace{0.5em}
        \STATE --- \textbf{\textsc{FedFW}:}  \hspace{0.9cm}   $\g_i^t = \frac{1}{n} \nabla f_i (\x_i^t) + \lambda_t (\x_i^t - \xb^t)$
        \vspace{0.5em}
        \STATE --- \textbf{\textsc{FedFW$\boldsymbol{+}$}:} \hspace{0.6cm}   $\mycolorOne{\y_i^{t+1} = \y_i^t +\, \lambda_0 (\x_i^t - \xb^t)}$ \\  
        \phantom{ --- \textbf{\textsc{FedFW-sto}:}\ }  $\mycolorOne{\g_i^t = \frac{1}{n} \nabla f_i (\x_i^t) + \lambda_t (\x_i^t - \xb^t) \mycolorOne{\,+\,\y_i^{t+1}}}$
        \vspace{0.5em}
        \STATE --- \textbf{\textsc{FedFW-sto}:}  \ \ 
 $\mycolorTwo{\bd_i^{t+1} = (1- \rho_t)\bd_i^t +\, \rho_t \frac{1}{n} \nabla f_i (\x_i^t, \omega_i^t)}$ \\
 \phantom{ --- \textbf{\textsc{FedFW-sto}:}\ \ }$\mycolorTwo{\g_i^t = \bd_i^{t+1} +\, \lambda_t (\x_i^t - \xb^t) } $
         \vspace{0.25em}    
        \STATE $\s_i^t = \arg\min \{ \ip{\g_i^t}{\x} : \x \in \dom \}$
        \STATE $\x_i^{t+1} = (1-\eta_t) \x_i^t + \eta_t \s_i^t$
        \STATE Client communicates $\s_i^t$ to the server. 
        \ENDFOR
        \vspace{0.25em}
        \STATE --- \textbf{Server}-level aggregation ------------------------------------
        \STATE $\xb^{t+1} =(1-\eta_t) \xb^{t} +\eta_t \left( \frac{1}{n} \sum_{i=1}^n \s_i^t\right)$
        \vspace{0.15em}
        \STATE Server communicates $\xb^{t+1}$ to the clients.
        \vspace{0.35em}
\ENDFOR
\end{algorithmic}
\end{algorithm}

\subsection{Convergence Guarantees}
\label{section:guarantees}

This section presents the convergence guarantees of \textsc{FedFW}. 
We begin with the guarantees for problems with a smooth and convex objective function. 

\begin{theorem} \label{theorem:sec:algo. and convergence anal.: convex}
Consider problem \eqref{eqn:sec:intro:problem def} with $L$-smooth and convex loss functions $f_i$. Then, estimation $\xb^t$ generated by \textsc{FedFW} with step-size $\eta_t = \frac{2}{t+1}$ and  penalty parameter $\lambda_t = \lambda_0 \sqrt{t+1}$ for any $\lambda_0 > 0$ satisfies
\begin{align}
    F(\xb^t)-F(\xb^*) \leq \mathcal{O}(t^{-1/2}).
\end{align}
\end{theorem}

\begin{remark}
Our proof is inspired by the analysis in \cite{yurtsever2018conditional}. 
However, a distinction lies in how the guarantees are expressed. 
In \cite{yurtsever2018conditional}, the authors demonstrate the convergence of $\x_i^t$ towards a solution by proving that both the objective residual and the distance to the feasible set converge to zero. 
In contrast, we establish the convergence of $\xb^t$, representing a feasible point, focusing only on the objective residual. 
We present detailed proof in the supplementary material. 
\end{remark} 

It is worth noting that the convergence guarantees of \textsc{FedFW} are slower compared to those of existing unconstrained or projection-based FL algorithms. 
For instance, in the smooth convex setting with full gradients, \textsc{FedAvg} \cite{mcmahan2017communication} achieves a rate of $\mathcal{O}(t^{-1})$ in the objective residual. 
In a convex composite problem setting, \textsc{FedDA} \cite{yuan2021federated} converges at a rate of $\mathcal{O}(t^{-2/3})$. 
While \textsc{FedFW} guarantees a slower rate of $\mathcal{O}(t^{-1/2})$, it is important to highlight that \textsc{FedFW} employs cheap linear minimization oracles. 

Next, we present the convergence guarantees of \textsc{FedFW} for non-convex problems. 
For unconstrained non-convex problems, the gradient norm is commonly used as a metric to demonstrate convergence to a stationary point. 
However, this metric is not suitable for constrained problems, as the gradient may not approach zero if the solution resides on the boundary of the feasible set. 
To address this, we use the following gap function, standard in FW analysis \cite{lacoste2016convergence}:
\begin{equation}\label{eq:FW-gap}
    \gap(\x) := \max_{\mathbf{u} \in \dom} ~
     \ip{\nabla F(\x)}{ \x-\mathbf{u} }.
\end{equation}
It is straightforward to show that $\gap(\x)$ is non-negative for all $\x \in \dom$, and it attains zero if and only if $\x$ is a first-order stationary point of Problem~\eqref{eqn:sec:intro:problem def}. 

\begin{theorem}\label{theorem:sec:algo. and convergence anal.: non-convex}
Consider problem \eqref{eqn:sec:intro:problem def} with $L$-smooth loss functions $f_i$. Suppose the sequence $\xb^t$ is generated by \textsc{FedFW} with the fixed step-size $\eta_t=T^{-{2}/{3}}$, and  penalty parameter $ \lambda_t=\lambda_0 T^{{1}/{3}}$ for an arbitrary $\lambda_0 > 0$. %
Then, 
\begin{align}
    \min_{1\leq t \leq T} ~ \gap(\xb^t) \leq \mathcal{O}(T^{-1/3}).
\end{align}
\end{theorem}

\begin{remark}
We present the proof in the supplementary material. Our analysis introduces a novel approach, as \cite{yurtsever2018conditional} does not explore non-convex problems. While our focus is primarily on problems \eqref{eqn:sec:intro:problem def} and \eqref{eqn:sec:algo. and convergence anal.:composite equivalent}, our methodology can be used to derive guarantees for a broader setting of minimization of a smooth non-convex function subject to affine constraints over a convex and compact set. 
\end{remark}

As with our previous results, the convergence rate in the non-convex setting is slower compared to \textsc{FedAvg}, which achieves an $\mathcal{O}(t^{-1/2})$ rate in the gradient norm (note the distinction between the gradient norm and squared gradient norm metrics). 
For composite FL problems with a non-convex smooth loss and a convex non-smooth regularizer, \textsc{FedDR} \cite{tran2021feddr} achieves an $\mathcal{O}(t^{-1/2})$ rate in the norm of a proximal gradient mapping. 
In contrast, our guarantees are in terms of the Frank-Wolfe (FW) gap. 
To our knowledge, \textsc{FedDA} does not offer guarantees in the non-convex setting.

\subsection{Privacy and Communication Benefits}
\label{subseq:Privacy and Communication Overheads}

\textsc{FedFW} offers low communication overhead since the communicated signals are the extreme points of $\dom$, which typically have low dimensional representation. 
For example, if $\dom$ is $\ell_1$ (resp., nuclear) norm-ball, then the signals $\mathbf{s}_i$ are 1-sparse (resp., rank-one).
Additionally, linear minimization is a nonlinear oracle, the reverse operator of which is highly ill-conditioned. 
Retrieving the gradient from its linear minimization output is generally unfeasible. 
For example, if $\dom$ is the $\ell_1$ norm-ball, then $\mathbf{s}_i$ merely reveals the sign of the maximum entry of the gradient. 
In the case of a box constraint, $\mathbf{s}_i$ only reveals the gradient signs. 
For the nuclear norm-ball, $\mathbf{s}_i$ unveils only the top eigenvectors of the gradient. 
Furthermore, \textsc{FW} is robust against additive and multiplicative errors in the linear minimization step \cite{jaggi2013revisiting}; consequently, we can introduce noise to augment data privacy without compromising the convergence guarantees. 

In a simple numerical experiment, we demonstrate the privacy benefits of communicating linear minimization outputs instead of gradients. 
This experiment is based on the Deep Leakage algorithm \cite{zhu2019deep} using the CIFAR100 dataset. 
Our experiment compares reconstructed images (\textit{i.e.,} leaked data points) obtained from shared gradients versus shared linear minimization outputs, under $\ell_1$ and $\ell_2$ norm constraints. 
\Cref{fig:privacy:compare} displays the final reconstructed images alongside the Peak Signal-to-Noise Ratio (PSNR) across iterations. 
It is evident that reconstruction via linear minimization oracles, particularly under the $\ell_1$ ball constraint, is significantly more challenging than raw gradients. 

\begin{figure}[!th]
    \centering
    \begin{subfigure}{0.49\textwidth}
        \includegraphics[width=\textwidth]{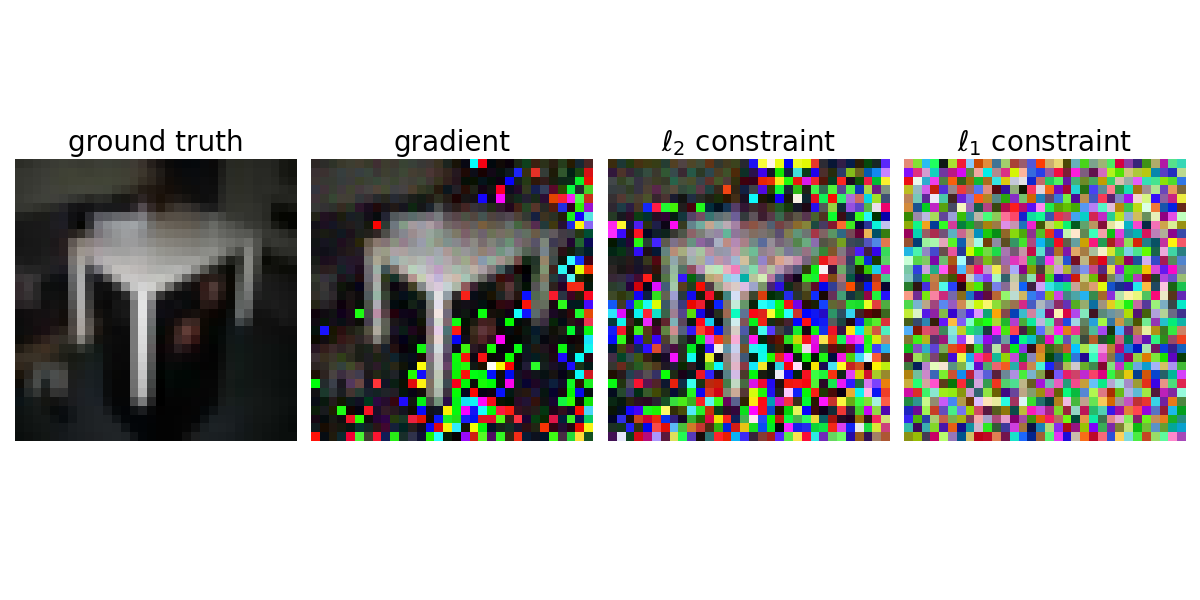}\vspace{-2em}
        \caption{}
        \label{fig:sub1}
    \end{subfigure}
    \begin{subfigure}{0.49\textwidth}
        \includegraphics[width=\textwidth]{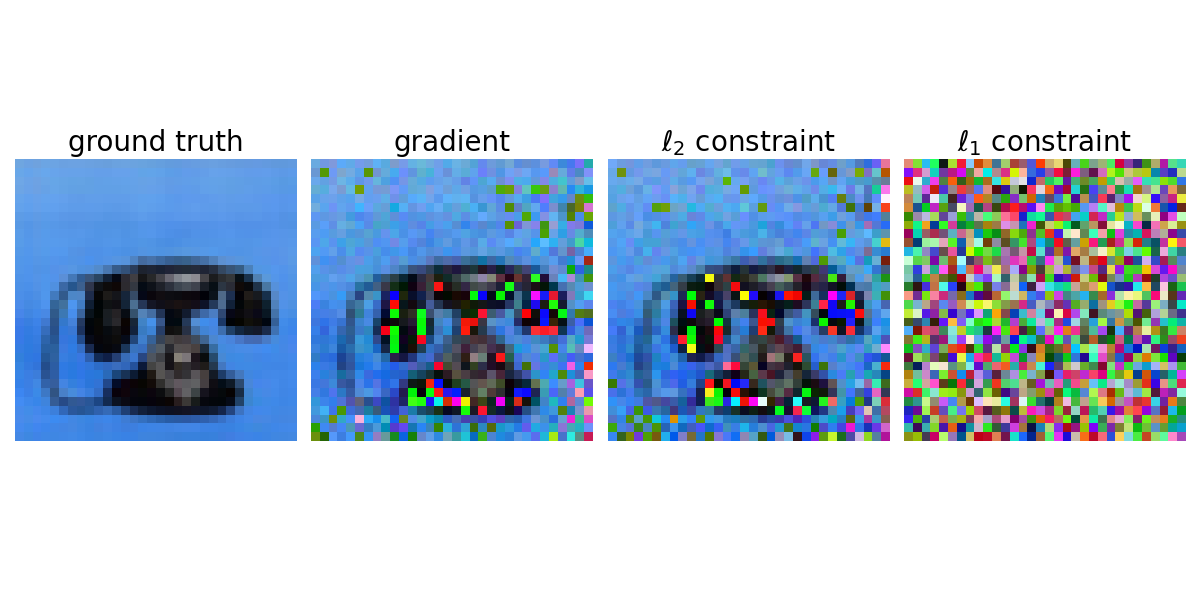}\vspace{-2em}
        \caption{}
        \label{fig:sub2}
    \end{subfigure}
    \\
    \begin{subfigure}{0.49\textwidth}
        \includegraphics[width=\textwidth]{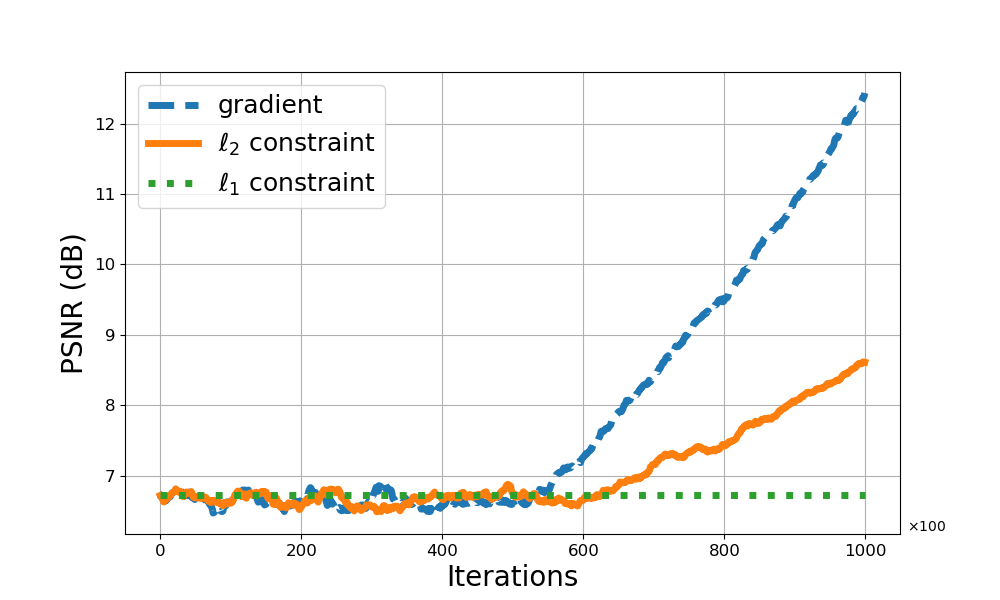}
        \caption{}
        \label{fig:sub3}
    \end{subfigure}
    \begin{subfigure}{0.49\textwidth}
        \includegraphics[width=\textwidth]{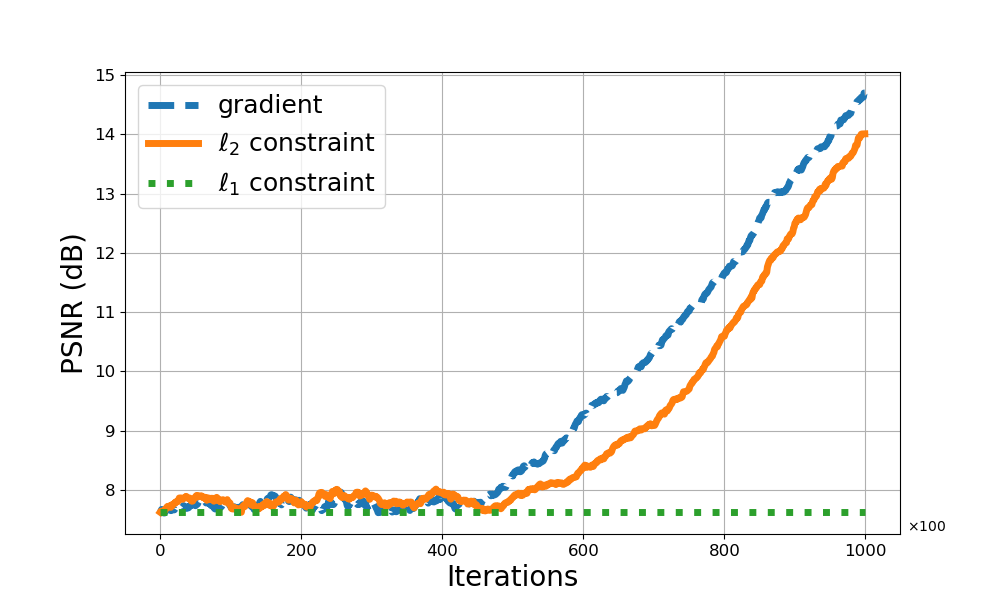}
        \caption{}
        \label{fig:sub4}
    \end{subfigure}
    \caption{Privacy benefits of sharing linear minimization outputs vs gradients. The Deep Leakage Algorithm can recover CIFAR-100 data points from shared gradients. Sharing linear minimization outputs enhances privacy. (a) and (b) compares reconstructions from gradients and LMO outputs with $\ell_2$ and $\ell_1$-norm ball constraints after $10^5$ iterations for two different data points. (c) and (d) present the reconstruction PSNR as a function of iterations for the corresponding images.}
    \label{fig:privacy:compare}
\end{figure}

\section{Design Variants of \textsc{FedFW}}
\label{sec:design-variants}
This section discusses several design variants and extensions of \textsc{FedFW}.
\subsection{\textsc{FedFW} with stochastic gradients}
Consider the following stochastic problem template:
\begin{equation}\label{eqn:sec:intro:problem def sto}
    \min_{\x \in \dom} ~ F(\x):= \avgn \mathbb{E}_{\omega_i} \big[ f_i(\x, \omega_i) \big].
\end{equation}
Here, $\omega_i$ is a random variable with an unknown distribution $\mathcal{P}_i$. The client loss function $f_i(\x) := \mathbb{E}_{\omega_i} \big[ f_i(\x, \omega_i) \big]$ is defined as the expectation over this unknown distribution; hence we cannot compute its gradient. We design \textsc{FedFW-sto} for solving this problem. 

We assume that at each iteration, every participating client can independently draw a sample $\omega_i^t$ from their distribution $\mathcal{P}_i$. 
$\nabla f_i(\x,\omega_i^t)$ serves as an unbiased estimator of $\nabla f_i(\x)$. 
Additionally, we adopt the standard assumption that the estimator has bounded variance. 
\begin{assumption}[Bounded variance]
\label{bounded-grad-noise}
Let $\nabla f_i(\x, \omega_i)$ denote the stochastic gradients. We assume that it satisfies the following condition for some $\sigma < \infty$:
\begin{equation}
    \mathbb{E}_{\omega_i} \Big[ \bnsq{ \nabla f_i(\x, \omega_i) - \nabla f_i(\x) } \Big] \leq \sigma^2.
\end{equation}
\end{assumption}

Unfortunately, FW does not readily extend to stochastic settings by replacing the gradient with an unbiased estimator of bounded variance. 
Instead, adapting FW for stochastic settings, in general, requires a variance reduction strategy. 
Inspired by \cite{locatello2019stochastic,mokhtari2020stochastic}, we employ the following averaged gradient estimator to tackle this challenge. 
We start by $\bd_i^0 = \boldsymbol{0}$, and iteratively update 
\begin{equation}
\bd_i^{t+1} = (1- \rho_t)\bd_i^t +\, \rho_t \frac{1}{n} \nabla f_i (\x_i^t, \omega_i^t), 
\end{equation}
for some $\rho_t \in (0,1]$. 
\textsc{FedFW-sto} uses $\bd_i^{t+1}$ in place of the gradient in the linear minimization step; pseudocode is shown in \Cref{algo_all}. 
Although $\bd_i^{t+1}$ is not an unbiased estimator, it offers the advantage of reduced variance. 
The balance between bias and variance can be adjusted by modifying $\rho_t$, and the analysis relies on finding the right balance, reducing variance sufficiently while maintaining the bias within tolerable limits. 

\begin{theorem} \label{theorem:sto convex}
Consider problem \eqref{eqn:sec:intro:problem def sto} with $L$-smooth and convex loss functions $f_i$. 
Suppose \Cref{bounded-grad-noise} holds. 
Then, the sequence $\xb^t$ generated by \textsc{FedFW-sto} in \Cref{algo_all}, with step-size $\eta_t = \frac{9}{t+8}$, penalty parameter $\lambda_t = \lambda_0 \sqrt{t+8}$ for an arbitrary $\lambda_0 > 0$, and $\rho_t =\frac{4}{(t+7)^{2/3}}$ satisfies
\begin{equation}
\mathbb{E} [ F(\xb^t) ] - F(\x^*) 
  \leq \mathcal{O}(t^{-1/3}).
\end{equation}
\end{theorem}
\begin{remark}
Our analysis in this setting is inspired by \cite{locatello2019stochastic}; however, we establish the convergence of the feasible point $\xb^t$. 
This differs from the guarantees in \cite{locatello2019stochastic}, which demonstrate the convergence of $\x_i^t$ towards a solution by proving that both the expected objective residual and the expected distance to the feasible set converge to zero. 
We present the detailed proof in the supplementary material.
\end{remark}

In the smooth convex stochastic setting, \textsc{FedAvg} achieves a convergence rate of $\mathcal{O}(t^{-1/2})$. 
This rate also applies to \textsc{FedDA} when addressing composite convex problems. 
Additionally, under the assumption of strong convexity, \textsc{Fast-FedDA} \cite{bao2022fast} achieves an accelerated rate of $\mathcal{O}(t^{-1})$. 
In comparison, \textsc{FedFW-sto} converges with $\mathcal{O}(t^{-1/3})$ rate; however, it benefits from the use of inexpensive linear minimization oracles.

\subsection{\textsc{FedFW} with partial client participation} \label{FedFW_with_pp}

A key challenge in FL is to tackle random device participation schedules. 
Unlike a classical distributed optimization scheme, in most FL applications, clients have some autonomy and are not entirely controlled by the server. 
Due to various factors, such as network congestion or resource constraints, clients may intermittently participate in the training process. 
This obstacle can be tackled in \textsc{FedFW} by employing a block-coordinate Frank-Wolfe approach \cite{lacoste2013block}. 
Given that the domain of problem \eqref{eqn:sec:algo. and convergence anal.:composite equivalent} is block-separable, we can extend our \textsc{FedFW} analysis to block-coordinate updates. 

Suppose that in every round $t$, the client $i$ participates in the training procedure with a fixed probability of $\texttt{p}_i \in (0,1]$. 
For simplicity, we assume the participation rate is the same among all clients, \textit{i.e.}, $\texttt{p}_1 = \ldots = \texttt{p}_n := \texttt{p}$, but non-uniform participation can be addressed similarly. 
Instate the convex optimization problem described in \Cref{theorem:sec:algo. and convergence anal.: convex} but with the random client participation scheme. %
At round $t$, the training procedure follows the same as in \Cref{algo_all} for the participating clients, and $\x_{i}^{t+1} = \x_{i}^t$ for the non-participants. 
Then, the estimation $\xb^t$ generated with the step-size $\smash \eta_t = \frac{2}{\texttt{p}(t-1)+2}$ and penalty parameter $ \scriptstyle \lambda_t = \lambda_0 \sqrt{ \texttt{p}(t-1)+2}$ converges to a solution with rate
\begin{align}
    \mathbb{E}[F(\xb^t) - F(\x^*)] \leq \mathcal{O}\big((\texttt{p}\,t)^{-1/2}\big).
\end{align}

Similarly, if we consider the non-convex setting of \Cref{theorem:sec:algo. and convergence anal.: non-convex} with randomized client participation, and use the block-coordinate \textsc{FedFW} with step-size $\eta_t={(\texttt{p}T+1)}^{-\frac{2}{3}}$, and  penalty parameter $\lambda_t=\lambda_0 {(\texttt{p}T+1)}^{\frac{1}{3}}$, we get
\begin{align}
\min_{1\leq t \leq T} ~  \mathbb{E}  [ \gap(\xb^t) ] \leq \mathcal{O}\big((\texttt{p}\,T)^{-1/3}\big).
\end{align}
The proofs are provided in the supplementary material.

\subsection{\textsc{FedFW} with split constraints for stragglers} \label{sec:stragglers}

FL systems are frequently implemented across heterogeneous pools of client hardware, leading to the `straggler effect'-- delays in execution resulting from clients with less computation or communication speeds. 
In \textsc{FedFW}, we can mitigate this issue by assigning tasks to straggling clients more compatible with their computational capabilities. 
Theoretically, this adjustment can be achieved through certain special reformulations of the problem defined in \eqref{eqn:sec:algo. and convergence anal.:composite equivalent}. 
Specifically, the constraint $\X \in \dom^n$ can be refined to $\X \in \bigcap_{i=1}^n \dom_i$, where $\bigcap_{i=1}^n \dom_i = \dom$. 
This modification does not affect the solution set, due to the consensus constraint.

In general, in \textsc{FedFW}, most of the computation occurs during the linear minimization step. 
Suppose that the resources of the client $i$ are limited, particularly for arithmetic computations. 
In this case, we can select $\dom_i$ as a superset of $\dom$ where linear minimization computations are more straightforward. 
For instance, a Euclidean (or Frobenius) norm-ball encompassing $\dom$ could be an excellent choice. 
Then, $\s_i^t$ becomes proportional to the negative of $\g_i^t$ with appropriate normalization based on the radius of $\dom_i$, facilitating computation with minimal effort. 
On the other hand, if the primary bottleneck is communication, we might opt for $\dom_i$ characterized by sparse extreme points, such as an $\ell_1$-norm ball containing $\dom$ or by low-rank extreme points like those in a nuclear-norm ball. 
This strategy results in sparse (or low-rank) $\s_i^t$, thereby streamlining communication.

\subsection{\textsc{FedFW} with augmented Lagrangian}

\textsc{FedFW} employs a quadratic penalty strategy to handle the consensus constraint. 
We also propose an alternative variant, \textsc{FedFW+}, which is modeled after the augmented Lagrangian strategy in \cite{yurtsever2019conditional}. 
The pseudocode for \textsc{FedFW+} is presented in \Cref{algo_all}.
We compare the empirical performance of \textsc{FedFW} and \textsc{FedFW+} in \Cref{section: Numerical Experiments}.  
The theoretical analysis of \textsc{FedFW+} is omitted here; for further details, we refer readers to \cite{yurtsever2019conditional}.

\section{Numerical Experiments}
\label{section: Numerical Experiments}

In this section, we evaluate and compare the empirical performance of our methods against \textsc{FedDR}, which serves as the baseline algorithm, on the convex multiclass logistic regression (MCLR) problem and the non-convex tasks of training convolutional neural networks (CNN) and deep neural networks (DNN).
For each problem, we consider two different choices for the domain $\dom$: namely the $\smash{\ell_1}$ and $\smash{\ell_2}$ ball constraints, each with a radius of $10$. 
We assess the models' performance based on validation accuracy, validation loss, and the Frank-Wolfe gap \eqref{eq:FW-gap}. 
To evaluate the effect of data heterogeneity, we conducted experiments using both IID and non-IID data distributions across clients. 
The code for the numerical experiments can be accessed via \url{https://github.com/sourasb05/Federated-Frank-Wolfe.git}. 

\paragraph{Datasets.} 
We use several datasets in our experiments: 
MNIST \cite{lecun1998gradient}, CIFAR-10 \cite{krizhevsky2009learning}, EMNIST \cite{cohen2017emnist}, and a synthetic dataset generated as described in \cite{t2020personalized}. 
Specifically, the synthetic data is drawn from a multivariate normal distribution, and the labels are computed using softmax functions. 
We create data points of $60$ features and from $10$ different labels. 
For all datasets, we consider both IID and non-IID data distributions across the clients. 
In the non-IID scenario, each user has data from only 3 labels. 
We followed this rule for the synthetic data, as well as MNIST, CIFAR$10$, and EMNIST-$10$. 
For EMNIST-$62$, each user has data from 20 classes, with unequal distribution among users.

\subsection{Comparison of algorithms in the convex setting}

We tested the performance of the algorithms on the strongly convex MCLR problem using the MNIST and CIFAR-10 datasets as well as the synthetic dataset. 
\Cref{table:strongly-convex} presents the test accuracy results for the algorithms with IID and non-IID data distributions, and for two different choices of $\dom$. 
In these experiments, we simulated FL with 10 clients, all participating fully ($\texttt{p}=1$). 
We ran the algorithms for $100$ communication rounds, with one local iteration per round.

\subsection{Comparison of algorithms in the non-convex setting}

For the experiments in the non-convex setting we trained CNNs using the MNIST  dataset and a DNN with two hidden layers using the synthetic dataset. 
We considered an FL system with $10$ clients and full participation $(\texttt{p}=1)$. 
Similar to the previous case, we evaluated IID and non-IID data distributions as well as different choices of
$\dom$, and ran the methods for 100 communication rounds with a single local training step. 
\Cref{table:non-convex}  summarizes the resulting test accuracies.

\subsection{Comparison of algorithms in the stochastic setting}

Finally, we compared the performance of \textsc{FedFW-sto} against \textsc{FedDR} in the stochastic setting, where only stochastic gradients are accessible. 
For this experiment, we consider an FL network with  $100$ clients with full participation ($\texttt{p}=1$). 
Over this network, we trained the MCLR model using EMNIST-$10$, EMNIST-$62$, CIFAR$10$, and the synthetic dataset. 
We used a mini-batch size of $64$, one local iteration per communication round, and ran the algorithms for $300$ communication rounds. 
\Cref{table:stochastic} summarizes the test accuracies obtained in this experiment. 
\textsc{FedFW-sto} outperformed \textsc{FedDR} in our experiments in the stochastic setting. 

\begin{table*}[t!]
\centering
\caption{Comparison of algorithms on the convex MCLR problem with different datasets and choices of $\dom$. We consider both IID and non-IID data distributions. The numbers represent test accuracy.} \label{table:strongly-convex}
\scriptsize
\begin{adjustbox}{width=\textwidth}
\begin{tabular}{| c | c | c | c | c | c | c |}
\specialrule{0.2em}{0.4pt}{0.4pt}
\multicolumn{1}{|c|}{} & \multicolumn{3}{c|}{IID} & \multicolumn{3}{c|}{non-IID} \\ \cline{2-7}
\multicolumn{1}{|c|}{} & MNIST & Synthetic & CIFAR10 & MNIST & Synthetic & CIFAR10 \\ \cline{2-7}
& \multicolumn{6}{|c|}{$\ell_2$ constraint} \\ \hline 
\textsc{FedDR} & $\mathbf{89.59(\pm 0.0003)}$ & $78.24(\pm 0.007)$ & $\mathbf{39.95(\pm 0.001)}$ & $83.72(\pm 0.001)$ & $92.97(\pm 0.005)$ & $37.79(\pm 0.004)$ \\ \hline
\textsc{FedFW} & $86.96(\pm 0.009)$ & $\mathbf{80.20(\pm 0.01)}$ & $36.30(\pm 0.001)$ & $86.95(\pm 0.001)$ & $\mathbf{94.81(\pm 0.001)}$ & $\mathbf{38.13(\pm 0.003)}$ \\ \hline
\textsc{FedFW+} & $86.50(\pm 0.001)$ & $79.96(\pm 0.001)$ & $36.30(\pm 0.001)$ & $\mathbf{86.98(\pm 0.001)}$ & $94.56(\pm 0.009)$ & $37.20(\pm 0.004)$ \\  \specialrule{0.2em}{0.4pt}{0.4pt}
& \multicolumn{6}{|c|}{$\ell_1$ constraint} \\ \cline{1-7}
\textsc{FedDR} & $72.18(\pm 0.0004)$ & $79.00(\pm 0.004)$ & $\mathbf{23.25(\pm 0.00)}$ & $74.29(\pm 0.0)$ & $\mathbf{93.81(\pm 0.009)}$ & $24.77(\pm 0.0)$ \\ \hline
\textsc{FedFW} & $\mathbf{78.07(\pm 0.005)}$ & $81.63(\pm 0.01)$ & $21.86(\pm 0.003)$ & $\mathbf{80.54(\pm 0.0)}$ & $90.84(\pm 0.003)$ & $25.08(\pm 0.004)$ \\ \hline
\textsc{FedFW+} & $69.17(\pm 0.004)$ & $\mathbf{81.92(\pm 0.008)}$ & $21.99(0.006)$ & $71.32(\pm 0.002)$ & $91.20(\pm 0.006)$ & $\mathbf{25.16(\pm 0.008)}$ \\  
\specialrule{0.2em}{0.4pt}{0.4pt}
\end{tabular}
\end{adjustbox}
\end{table*}

\begin{table*}[t!]
\centering
\caption{Comparison of algorithms on the non-convex tasks. We train a CNN using MNIST, and a DNN with synthetic data. We consider IID and non-IID data distributions, and different choices of $\dom$. The numbers show test accuracy.} \label{table:non-convex}
\small
\begin{tabular}{| c | c | c | c | c |}
\specialrule{0.2em}{0.4pt}{0.4pt}
   & \multicolumn{2}{| c |}{IID} & \multicolumn{2}{| c |}{non-IID} \\
    \cline{2-5}
    & MNIST & Synthetic  & MNIST & Synthetic   \\
    \cline{2-5}
    & \multicolumn{4}{| c |}{{$\ell_2$ constraint}} \\

    \hline
    \textsc{FedDR} & $\mathbf{96.89(\pm 0.0009)}$ & $75.96(\pm 0.03)$ & $88.93(\pm 0.013)$ &  $93.85(\pm 0.009)$ \\ \hline
     \textsc{FedFW} & $95.87(\pm 0.01)$  & $81.70(\pm 0.008)$ & $\mathbf{92.70(\pm 0.002)}$ & $\mathbf{ 96.13(\pm 0.007)}$  \\ \hline
     \textsc{FedFW+} & $95.05(\pm 0.005)$ & $\mathbf{ 81.96(\pm 0.006)}$  &  $91.79(\pm 0.005)$ & $96.08(\pm 0.004)$ \\      
     \specialrule{0.2em}{0.4pt}{0.4pt}
     & \multicolumn{4}{| c |}{{$\ell_1$ constraint}} \\
     \cline{2-5} 
    \textsc{FedDR} & $11.72(\pm 0.0)$ & $\mathbf{78.59(\pm 0.006)}$ & $16.75(\pm 0.01)$ &  $\mathbf{ 93.51(\pm 0.006)}$  \\ \hline
    \textsc{FedFW} & $\mathbf{23.88(\pm 0.005)}$ & $75.52(\pm 0.008)$ & $\mathbf{37.62(\pm 0.008)}$ & $91.53(\pm 0.01)$  \\ \hline
     \textsc{FedFW+} & $20.40(\pm 0.003)$ & $76.44(\pm 0.004)$ & $36.27(\pm 0.006)$ & $91.96(\pm 0.003)$ \\
\specialrule{0.2em}{0.4pt}{0.4pt}
\end{tabular}
\end{table*}

\begin{table*}[h!]
\centering
\caption{Comparison of algorithms in the stochastic setting on the convex MCLR problem with different datasets and $\ell_2$ ball constraint. We consider both IID and non-IID data distributions. The numbers represent test accuracy.} \label{table:stochastic}
\small
\begin{tabular}{| c | c | c | c | c |}
\specialrule{0.2em}{0.4pt}{0.4pt}
\multicolumn{1}{|c|}{}  & \multicolumn{2}{c|}{IID} & \multicolumn{2}{c|}{non-IID} \\ \cline{2-5}
 & EMNIST-10 & EMNIST-62 & EMNIST-10 & EMNIST-62
\\ \hline
\textsc{FedDR} & $92.18(\pm 0.01)$ & $39.58(\pm 0.00)$ & $85.70(\pm 0.002)$ & $41.09(\pm 0.002)$ \\ \hline
\textsc{FedFW-sto} & $\mathbf{93.79(\pm 0.00)}$ & $\mathbf{41.22(\pm 0.00)}$ & $\mathbf{91.88(\pm 0.01)}$ & $\mathbf{60.51(\pm 0.002)}$ \\
\specialrule{0.2em}{0.4pt}{0.4pt} 
\multicolumn{1}{|c|}{} & \multicolumn{1}{|c|}{Synthetic}  & 
\multicolumn{1}{|c|}{CIFAR10}  &
\multicolumn{1}{|c|}{Synthetic}  & 
\multicolumn{1}{|c|}{CIFAR10} 
\\ \hline 
\textsc{FedDR} & \multicolumn{1}{|c|}{$67.68 (\pm 0.003)$} & $36.39(\pm 0.004)$ & \multicolumn{1}{|c|}{$84.70 (\pm 0.01)$} & $34.91(\pm 0.008)$ \\ \hline
\textsc{FedFW-sto} & \multicolumn{1}{|c|}{$\mathbf{72.01 (\pm 0.004)}$} & $\mathbf{38.52(\pm 0.01) }$& \multicolumn{1}{|c|}{$\mathbf{87.32 (\pm 0.003)}$} & $\mathbf{37.83(\pm 0.01)}$
\\  
\specialrule{0.2em}{0.4pt}{0.4pt}
\end{tabular}
\end{table*}

\newcommand{\FigWidthOne}{0.33}
\newcommand{\FigWidthTwo}{0.25}
\newcommand{\FigWidthThree}{0.33}
\newcommand{\FigWidthFour}{0.5}

\begin{figure}[!htb]
\centering

\begin{subfigure}[b]{0.49\textwidth}
  \centering
  \includegraphics[width=\textwidth]{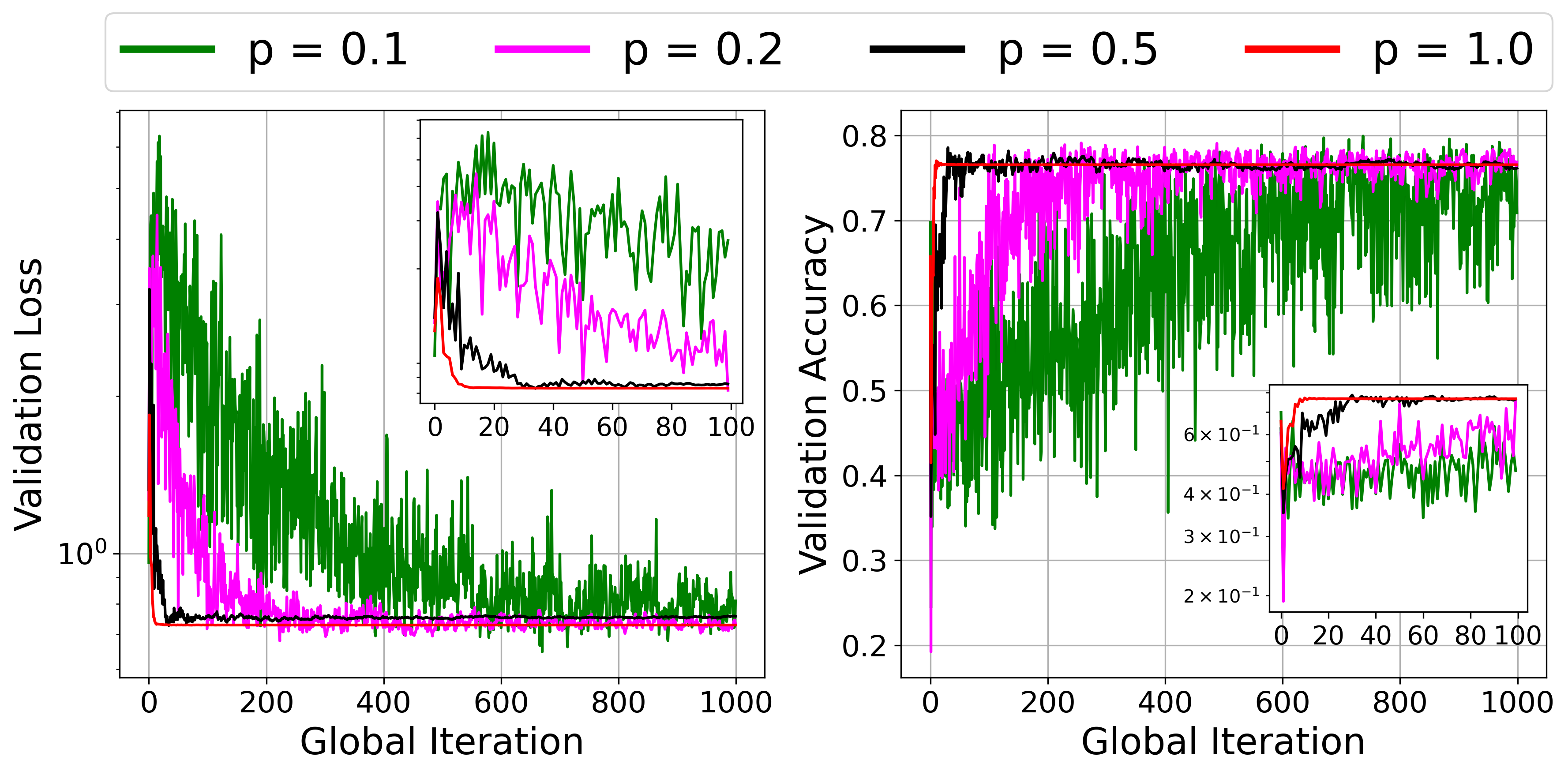}
  \caption{$\lambda_0 = 0.01$}
\end{subfigure}
\hfill %
\begin{subfigure}[b]{0.49\textwidth}
  \centering
  \includegraphics[width=\textwidth]{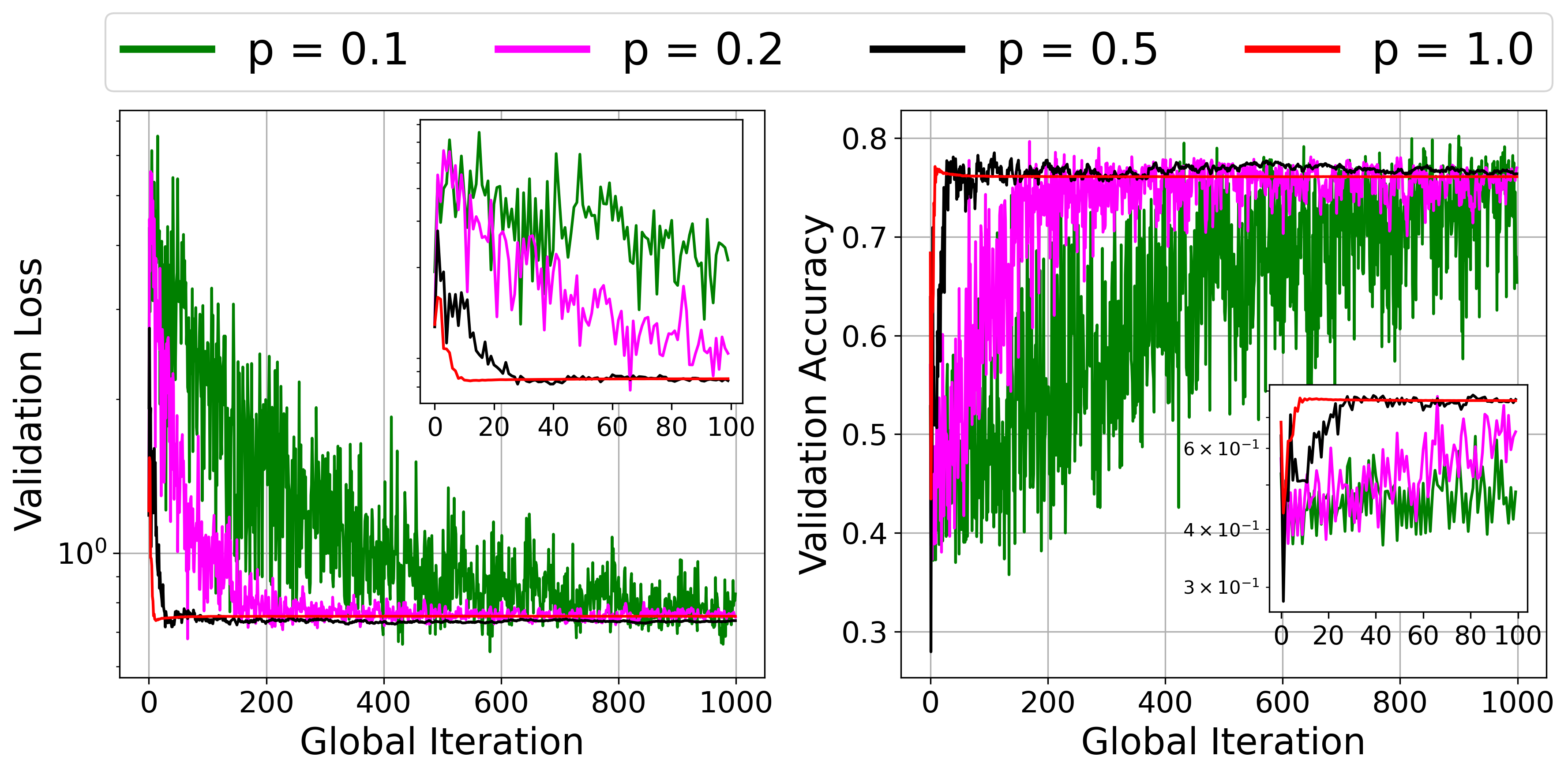}
  \caption{$\lambda_0 = 0.001$}
\end{subfigure}

\caption{Effect of participation $\texttt{p}$ on \textsc{FedFW}. The experiment was conducted with MCLR using synthetic data, an $\ell_1$ constraint, and two different choices of $\lambda_0$.}
\label{fig: partial-participation_MCLR}

\vspace{2em}

\begin{subfigure}[b]{0.49\textwidth}
  \centering
  \includegraphics[width=\linewidth]{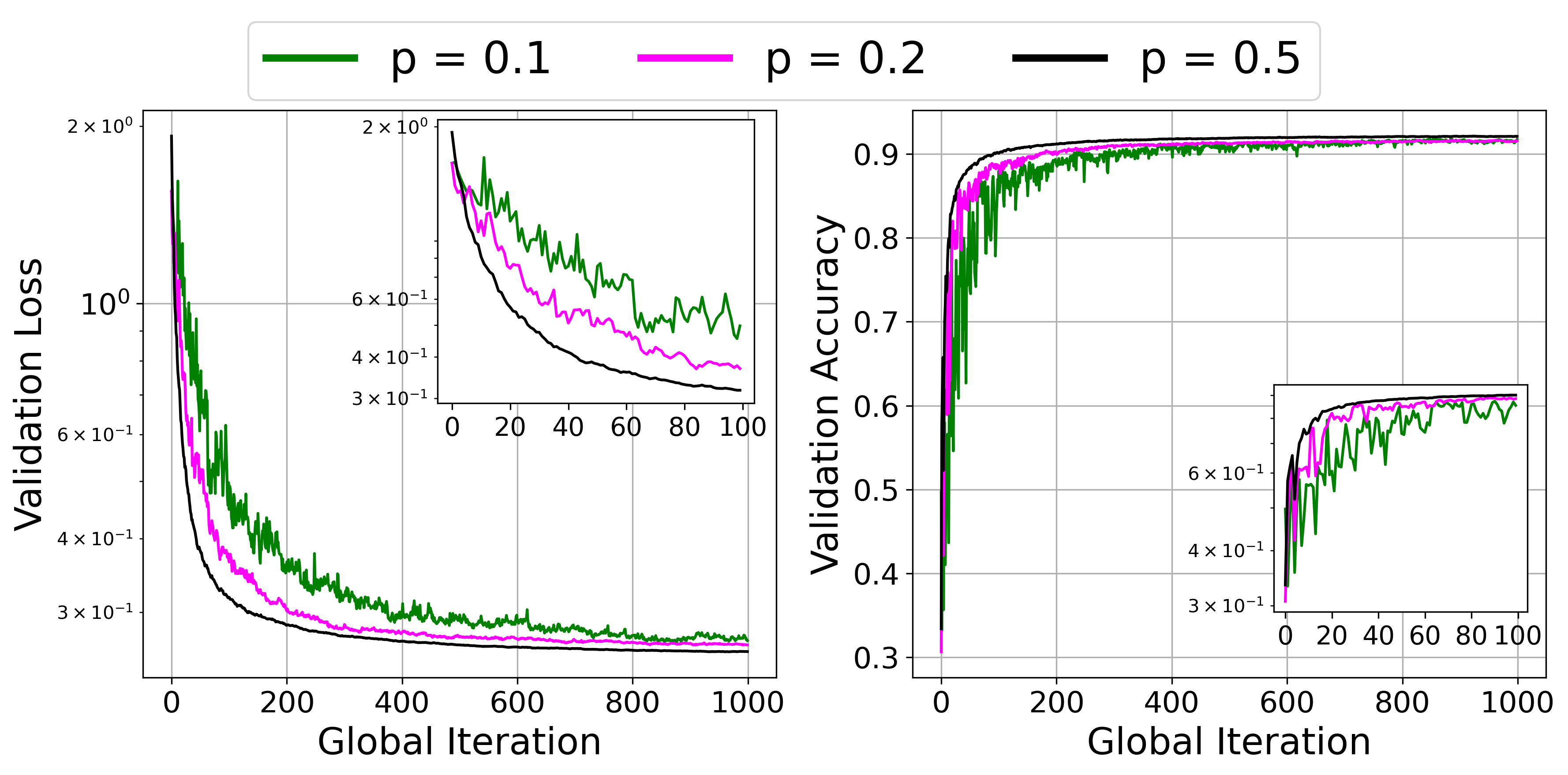}
  \caption{FedFW}
\end{subfigure}%
\hfill %
\begin{subfigure}[b]{0.49\textwidth}
  \centering
  \includegraphics[width=\linewidth]{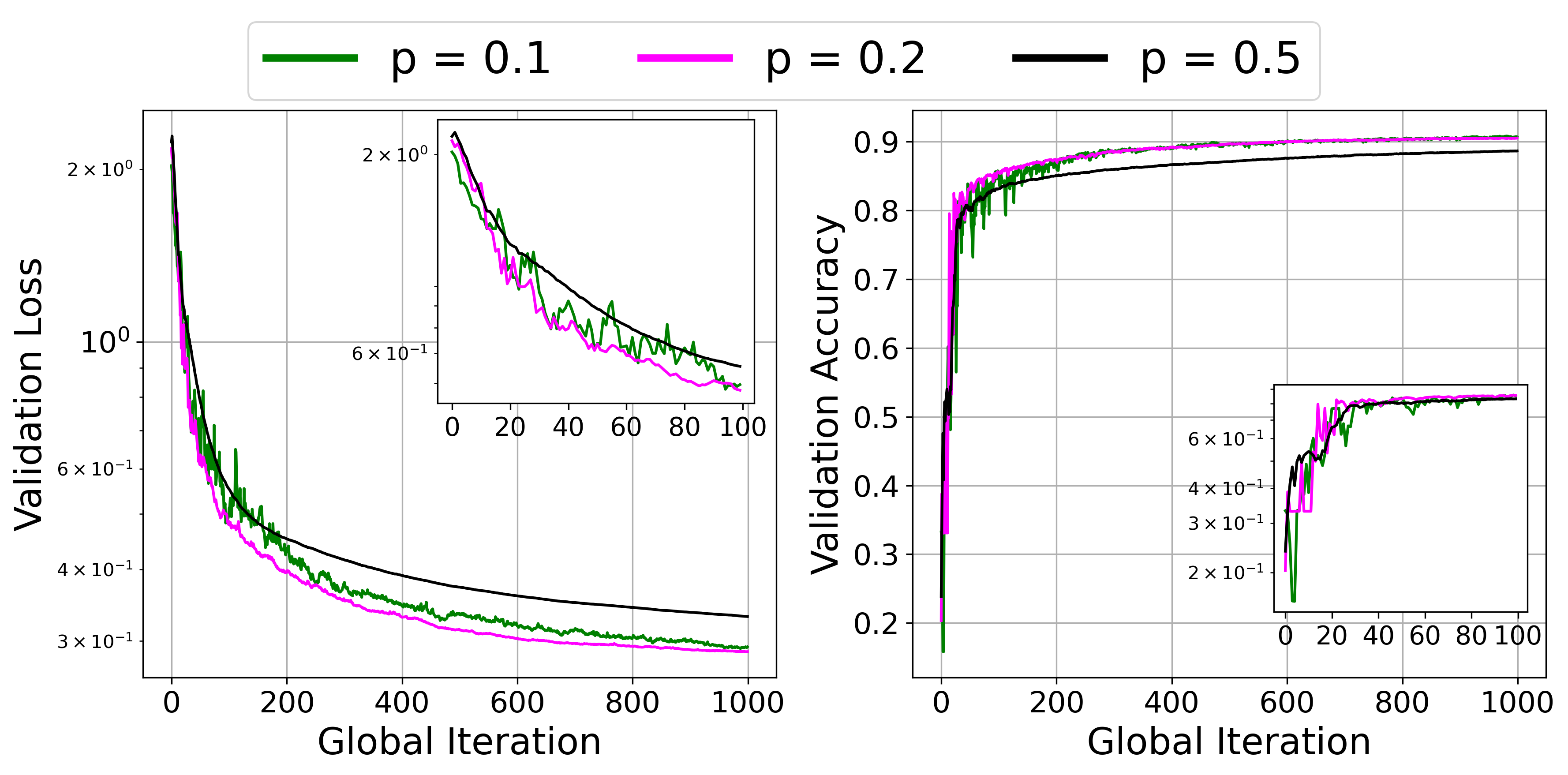}
  \caption{\textsc{FedFW+}}
\end{subfigure}%
\caption{Effect of participation $\texttt{p}$ on \textsc{FedFW} and \textsc{FedFW+}. We trained a DNN model using synthetic data, an $\ell_2$ constraint, and a fixed $\lambda_t = 10^{-3}$.}
\label{fig: partial-participation_DNN}

\end{figure}

\begin{figure}[p]
\centering

\begin{subfigure}{0.85\textwidth}
  \centering
  \includegraphics[width=1\linewidth]{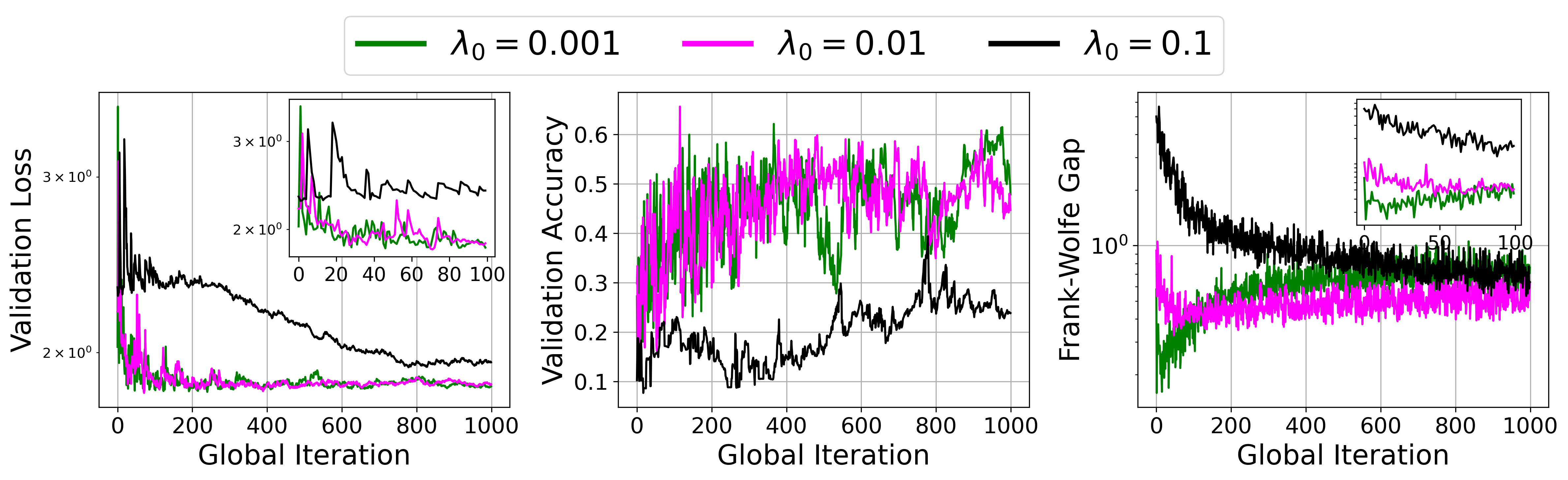}
  \caption{Convex MCLR with MNIST data, participation ratio $\texttt{p}=0.5$, and $\ell_1$ constraint.}
\end{subfigure}

\vspace{0.5em}

\begin{subfigure}{0.85\textwidth}
  \centering
  \includegraphics[width=1\linewidth]{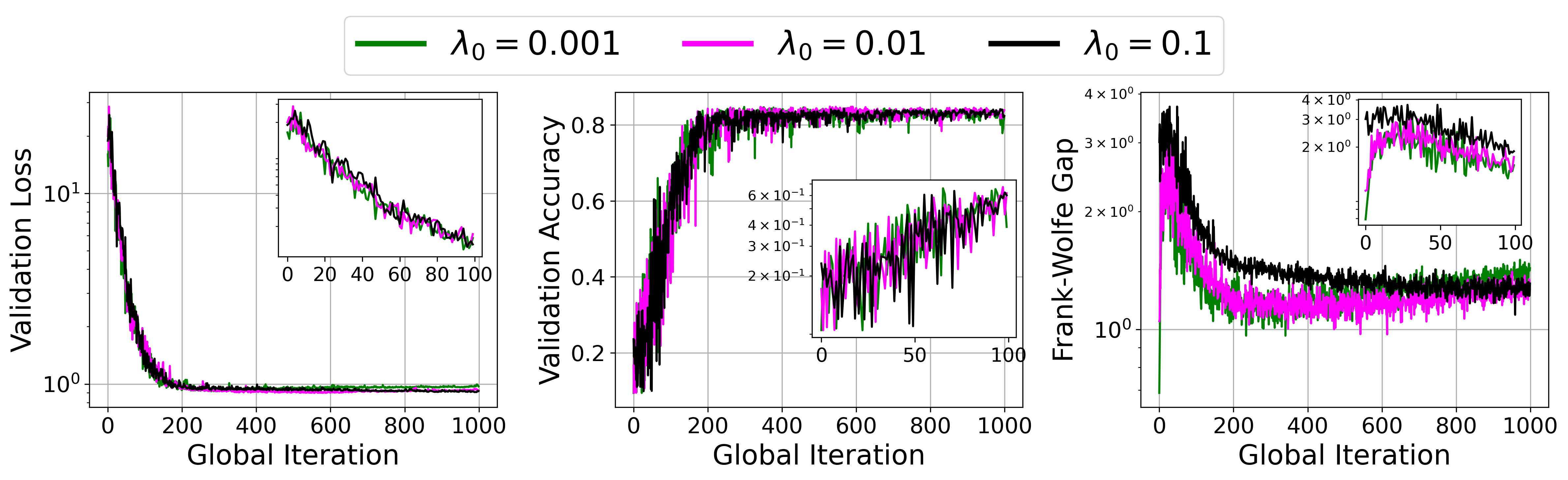}
  \caption{Convex MCLR with MNIST data, participation ratio $\texttt{p}=0.5$, and $\ell_2$ constraint.}
\end{subfigure}%

\vspace{0.5em}

\begin{subfigure}{0.85\textwidth}
  \centering
  \includegraphics[width=\linewidth]{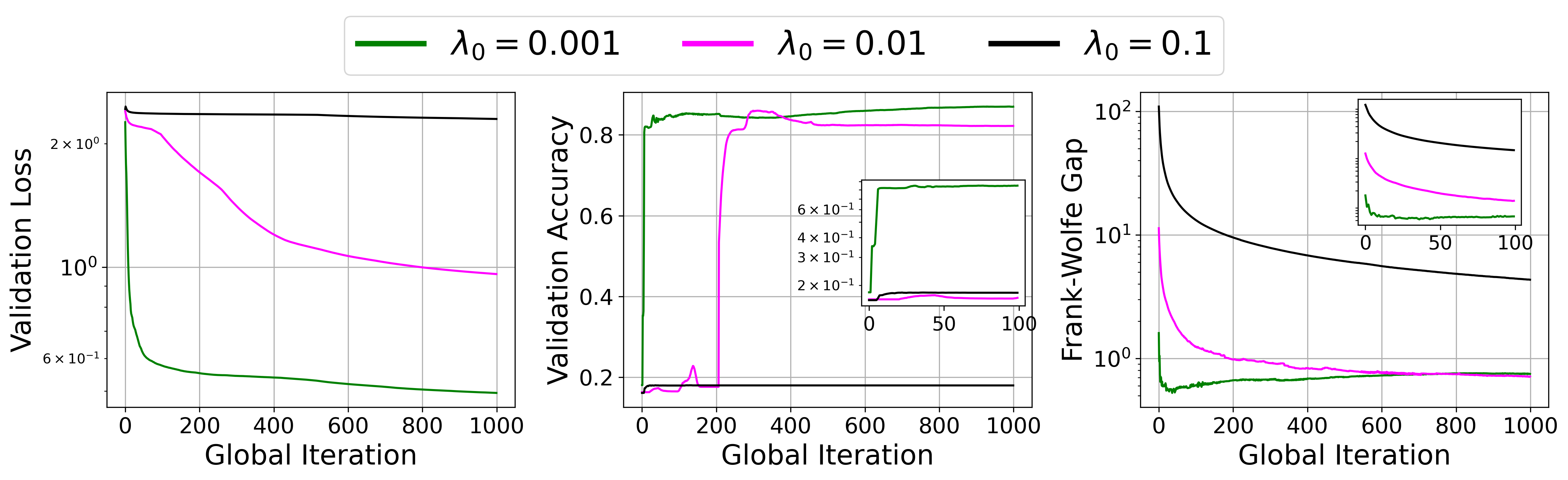}
  \caption{Non-convex DNN with synthetic data, full participation $\texttt{p}=1$, and $\ell_1$ constraint.}
\end{subfigure}%

\vspace{0.5em}

\begin{subfigure}{0.85\textwidth}
  \centering
  \includegraphics[width=\linewidth]{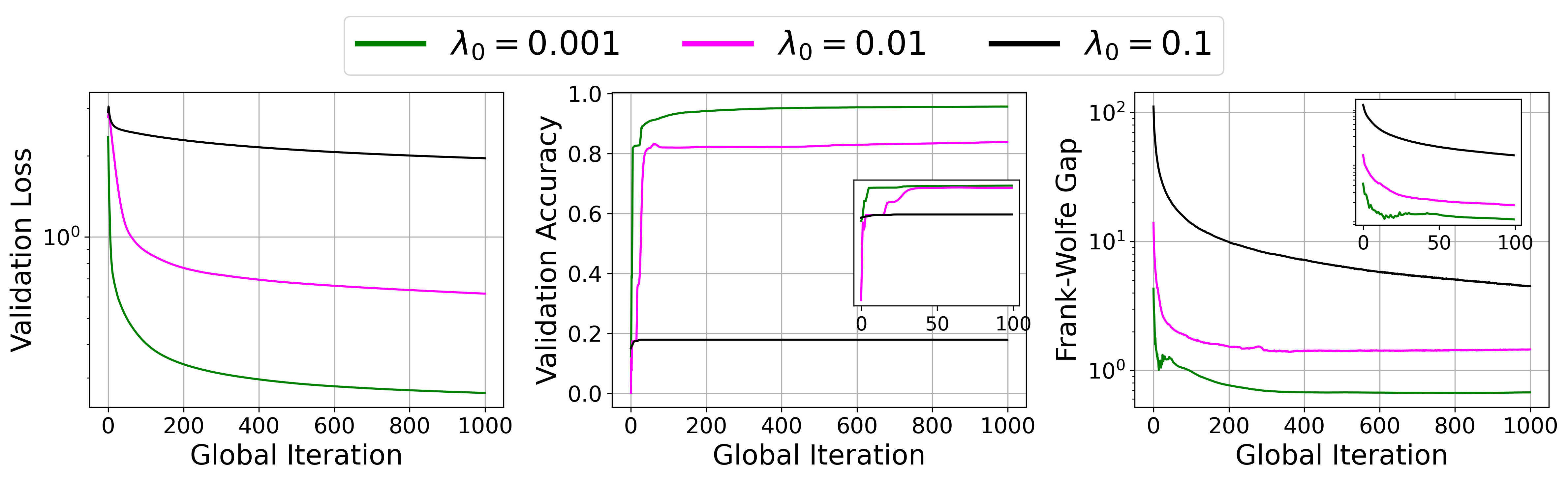}
  \caption{Non-convex DNN with synthetic data, full participation $\texttt{p}=1$, and $\ell_2$ constraint.}
\end{subfigure}%

\caption{Effect of the initial penalty ($\lambda_0$) on \textsc{FedFW}. (a) and (b) show the results for the convex setting, (c) and (d) demonstrates the non-convex setting.
}
\label{fig: hyperparameters_choice}
\end{figure}

\subsection{Impact of hyperparameters}

We conclude our experiments with an ablation study to investigate how varying hyperparameters impact the performance of \textsc{FedFW}. 

\paragraph{Impact of partial participation} \textbf{(}$\texttt{p}$\textbf{).}~
\Cref{fig: partial-participation_MCLR} shows the validation accuracy and loss of \textsc{FedFW} algorithm for synthetic data and MCLR model. 
\Cref{fig: partial-participation_DNN} depicts the validation accuracy and loss of \textsc{FedFW} and \textsc{FedFW+} algorithms for synthetic data and DNN model. 
Both convex and non-convex experiments show faster convergence for higher participation probability. 
It is worth mentioning that variations in $\lambda_0$ do not alter the influence of $\texttt{p}$. 
These observations are in accordance with the theoretical guarantees presented in \Cref{FedFW_with_pp}. 

\paragraph{Impact of initial penalty parameter} \textbf{(}$\lambda_0$\textbf{).}~
\Cref{fig: hyperparameters_choice} illustrates the effect of hyperparameters on the convergence of loss, Frank-Wolfe gap, and validation accuracy of the algorithms. 
A higher $\lambda_0$ leads to a larger gap in the initial iterations of the algorithm due to its regularization effect. 
In other words, increasing $\lambda_0$ enforces the update direction towards the consensus set, which in turn increases the gap value in the first iteration. 
The exact expressions for the constants in the convergence guarantees, which are detailed in the supplementary material, can guide the optimal choice of $\lambda_0$.

\section{Conclusions}
\label{section: Conclusions}

We introduced a FW-type method for FL and established its theoretical convergence rates. The proposed method, \textsc{FedFW}, guarantees $\smash{\mathcal{O}(t^{-1/2})}$ convergence rates when the objective function smooth and convex. If we remove the convexity assumption, the rate reduces to $\smash{\mathcal{O}(t^{-1/3})}$.
With access to only stochastic gradients, \textsc{FedFW} achieves an $\smash{\mathcal{O}(t^{-1/3})}$ convergence rate in the convex setting. 
Additionally, we proposed an empirically faster version of \textsc{FedFW} by incorporating an augmented Lagrangian dual update. 

We conclude with a brief discussion on the limitations of our work. 
The primary limitation of \textsc{FedFW} is its slower convergence rates compared to state-of-the-art FL methods. 
Developing a tighter bound for \textsc{FedFW}, with multiple local steps, is an area for future research. 
Additionally, the analysis of \textsc{FedFW+} is left to future work. 
Another important piece of future work is the convergence analysis of \textsc{FedFW-sto} for non-convex objectives. 
Finally, the development and analysis of an extension for asynchronous updates also remain as future work.

\subsection*{Acknowledgments}
This work was supported by the Wallenberg AI, Autonomous Systems and Software Program (WASP) funded by the Knut and Alice Wallenberg Foundation. 
We also acknowledge support from the Swedish Research Council under the grant registration number 2023-05476. 
The computations were enabled by the Berzelius resource provided by the Knut and Alice Wallenberg Foundation at the National Supercomputer Centre. 
Additionally, computations in an earlier version of this work were enabled by resources provided by the Swedish National Infrastructure for Computing (SNIC) at Chalmers Centre for Computational Science and Engineering (C3SE) partially funded by the Swedish Research Council through grant agreement no. 2018-05973. %
We appreciate the discussions with Yikun Hou on the numerical experiments and implementation. 
We acknowledge the use of OpenAI's ChatGPT for editorial assistance in preparing this manuscript.

\newpage
\appendix

\section{Challenges in Extending FW for FL}
\label{sec:counter-example}

The generic problem template for Federated Learning (FL) that we focus in this paper is as follows: 
\begin{equation}\label{supp:eqn:sec:intro:problem def}
    \min_{\x \in \dom} ~~ F(\x):= ~ \frac{1}{n} \sum_{i=1}^n f_i(\x),
\end{equation}
Our main goal in this paper is to develop a Frank-Wolfe algorithm (FW) for this problem template. 
To extend FW for FL, a natural attempt involves integrating local FW steps with a centralized aggregation step. 
Specifically, consider the following procedure, for $t = 1,\ldots, T$:
\begin{equation}\label{algo:failing-fw}
    \begin{aligned}
        & \s_i^t \in \arg\min_{\x \in \dom} ~ \ip{\nabla f_i(\xb^t)}{\x} & & \text{for~$i = 1,\ldots,n$} \\
        &  \x_i^{t+1} = (1-\eta_t) \xb^t + \eta_t \s_i^t & & \text{for~$i = 1,\ldots,n$}  \\
        &  \xb^{t+1} = \frac{1}{n} \sum_{i=1}^n \x_i^{t+1} \\
    \end{aligned}
\end{equation}
However, this algorithm fails to solve the following 1-dimensional problem: 
\begin{equation}\label{eqn:counter-example}
    \min_{x \in [-1,1]} ~~~ \frac{1}{2}(x-3)^2 + \frac{1}{2}(x+1)^2. 
\end{equation}
This is an instance of the model problem in \eqref{supp:eqn:sec:intro:problem def} with $f_1(x) = (x-3)^2$ and $f_2(x) = (x+1)^2$. 
It has a unique solution at $x^\star = 1$. 
Now, consider procedure \eqref{algo:failing-fw} initialized from $\bar{x}^1 = 0$. Then, $s_1^1 = 1$ and $s_2^1 = -1$, and for any $\eta_1 \in [0,1]$, we get
\begin{equation}
    \bar{x}^2 
    = \frac{1}{2} x_1^2 + \frac{1}{2} x_2^2 
    = (1-\eta_1) \bar{x}^1 + \eta_1 \frac{1}{2} (s_1^1 + s_2^1)
    = 0.
\end{equation}
Therefore, $x = 0$ is a fixed point for the procedure \eqref{algo:failing-fw} although the unique solution of \eqref{eqn:counter-example} is $x = 1$. We conclude that this method may fail. Similar arguments hold even if we run multiple local training steps before the aggregation.

\section{Convergence Analysis of \textsc{FedFW} in the Convex Case} \label{sec: appendix: theorem 1}

In this section, we provide the proof of Theorem~1 and its generalization to the sampled clients scenario.  
First, we mention some preliminary results.

\begin{lemma}\label{lem:sec: appendix: SmoothnessOfhhat} 
Recall the surrogate function $\hat{F}_t$ defined in the main paper:
\begin{equation*}
    \hat{F}_t(\mathbf{X}) = \frac{1}{n} \sum_{i=1}^n f_i(\mathbf{X} \mathbf{e}_i) + \frac{\lambda_t}{2} \dist^2(\mathbf{X},\C), 
\end{equation*}
Suppose each $f_i$ is $L$-smooth, meaning the gradients are Lipschitz continuous with parameter $L$. Then, $\hat{F}_t(\mathbf{X})$ is $\hat{L}_t$-smooth, where $\hat{L}_t = \frac{L}{n} + \lambda_t$.
\end{lemma}
\begin{proof}
Using the definition of $\hat{F}_t$ and $\nabla \hat{F}_t$ in the main paper, we have 
\begin{align}
    \|\nabla \hat{F}_t(\X)- \nabla \hat{F}_t(\Y) \|_F 
    & = \| \frac{1}{n}\sum_{i=1}^n \left( \nabla f_i(\x_i)-\nabla f_i(\y_i)\right) \cdot \mathbf{e}_i^\top + \lambda_t (\X - \bar{\X} - \Y + \bar{\Y}) \|_F 
    \nonumber \\
    & \leq \frac{1}{n} \| 
    \sum_{i=1}^n \left( \nabla f_i(\x_i)-\nabla f_i(\y_i)\right) \cdot \mathbf{e}_i^\top \|_F +\lambda_t \| \X - \bar{\X} - \Y + \bar{\Y} \|_F
     \nonumber \\
    & = \frac{1}{n} \sqrt{ \sum_{i=1}^n  \| 
      \nabla f_i(\x_i)-\nabla f_i(\y_i) \|_2^2} +\lambda_t \| \X - \bar{\X} - \Y + \bar{\Y} \|_F
    \nonumber \\
    & \leq \frac{L}{n} \sqrt{ \sum_{i=1}^n  \| 
      \x_i-\y_i \|_2^2} +\lambda_t \| \X - \bar{\X} - \Y + \bar{\Y} \|_F
    \nonumber \\
    & = \frac{L}{n} \| 
      \X-\Y \|_F +\lambda_t \| \X - \bar{\X} - \Y + \bar{\Y} \|_F
    \nonumber \\
    &=  
    \frac{L}{n}\|\X-\mathbf{Y}\|_F
    +
    \lambda_t \|(\X-\mathbf{Y})(I_n- \frac{1}{n} J_n) \|_F
    \nonumber \\
    &\leq 
    \frac{L}{n}\|\X-\mathbf{Y}\|_F
    +
    \lambda_t \|(\X-\mathbf{Y}) \|_F
    \nonumber \\
    &= 
    \big(\frac{L}{n}+\lambda_t \big) \|\X-\mathbf{Y}\|_F.
\end{align}
\end{proof}
where $J_n$ is the all-ones matrix of size $p \times n$ and in the $8^{th}$ line, we used submultiplicativity of the Frobenius norm and the fact that the operator norm of the centering matrix $(I_n- \frac{1}{n} J_n)$ equals to $1$.

\begin{lemma}[Boundedness of the gradient] 
\label{lem:background:def-G}
Consider $\hat{F}(\mathbf{X}) := \frac{1}{n} \sum_{i=1}^n f_i(\mathbf{X} \mathbf{e}_i)$ where $f_i(.)$ are $L$-smooth. Then, 
$\norm{\nabla \hat{F}(\mathbf{X})}_F \leq G := \frac{LD}{n} + \max_{\X^* \in \argmin{} \{F(\X)\}  } \{ \norm{\nabla \hat{F}(\mathbf{X}^*)} \}$ for all $\X \in {\dom}^n$.

\begin{proof}
Let us define $\X^* \in \argmin{\X \in \dom^n} \{F(\X)\} $. Using the smoothness and boundedness assumptions, we can write
\begin{align*}
\norm{\nabla \hat{F}(\mathbf{X})}_F 
&\leq \norm{\nabla \hat{F}(\mathbf{X}) -  \nabla \hat{F}(\mathbf{X}^*) }_F
+
\norm{\nabla \hat{F}(\mathbf{X}^*)}_F
\leq \frac{L}{n} \, \norm{\X - \X^* }_F
+
\norm{\nabla \hat{F}(\mathbf{X}^*)}_F
\leq \frac{LD}{\sqrt{n}}
+
\norm{\nabla \hat{F}(\mathbf{X}^*)}_F.
\end{align*}
By definition of $\hat{F}$, and since $\mathbf{X}^*$ is a feasible point, we get
\begin{equation*}
\norm{\nabla \hat{F}(\mathbf{X}^*)}_F 
= \norm{\nabla \hat{F}(\mathbf{x}^* \mathbf{1}^\top)}_F 
\leq \sqrt{n} \, \norm{\nabla F(\mathbf{x}^*) \, \mathbf{1}^\top}_F 
= n \, \norm{\nabla F(\mathbf{x}^*)},
\end{equation*}
where $\mathbf{1} \in \mathbb{R}^n$ is the vector of ones. 
We complete the proof by combining this with the previous inequality. 

{\hfill $\square$}
\end{proof}

\end{lemma}

\subsection{Proof of Theorem~1}
By \Cref{lem:sec: appendix: SmoothnessOfhhat}, $\hat{F}_{t}$ is $\hat{L}_t$-smooth. 
Therefore, we have
\begin{align}
	 \label{algo1-convex: eq1}
	\hat{F}_{t}(\X^{t+1})
	&\leq \hat{F}_{t} ({\Xt}) + 
	\ip{\nabla \hat{F}_{t} ({\Xt})}{\X^{t+1}-\Xt}
	+\frac{\Lhat}{2} \| \X^{t+1}-\Xt \|_F^2
	\nonumber \\
	&= \hat{F}_{t} ({\Xt}) + \eta_t
	\ip{\nabla \hat{F}_{t} ({\Xt})}{\mathbf{S}^t-\Xt }
	+\frac{\Lhat \eta_t^2}{2} \| \mathbf{S}^t-\Xt \|_F^2
	\nonumber \\
	&\leq  \hat{F}_{t} ({\Xt}) + \eta_t
	\ip{\nabla \hat{F}_{t} ({\Xt})}{\mathbf{S}^t-\Xt }
	+\frac{\Lhat \eta_t^2}{2} n D^2 \nonumber \\
	&\leq  \hat{F}_{t} ({\Xt}) + \eta_t
	\ip{\nabla \hat{F}_{t} ({\Xt})}{\X^*-\Xt }
	+\frac{\Lhat \eta_t^2}{2} n D^2,
\end{align}
where the second line follows the definition of $\X^{t+1}$, the third line holds since $\dom$ is compact, and the last line holds by definition of $\mathbf{S}^t$. Now, let us define $\hat{F}(\mathbf{X}) := \frac{1}{n} \sum_{i=1}^n f_i(\mathbf{X} \mathbf{e}_i)$. Then, we get
\begin{align}\label{algo1-convex: eq2}
    \ip{\nabla \hat{F}_t ({\Xt})}{\X^*-\Xt} 
    & = \ip{\nabla \hat{F}(\Xt) }{\X^*-\Xt} + \lambda_t \ip{ \Xt - \proj_\C(\Xt) }{\X^*-\Xt} \nonumber \\
    & \leq \hat{F}(\X^*)-\hat{F}(\Xt) + \lambda_t \ip{ \Xt - \proj_\C(\Xt) }{\X^*-\Xt} \nonumber \\
    & = \hat{F}(\X^*)-\hat{F}(\Xt) + \lambda_t \ip{ \Xt - \proj_\C(\Xt) }{\X^*-\proj_\C(\Xt)} - \lambda_t \|\Xt - \proj_\C(\Xt)\|_F^2 \nonumber \\
    & \leq \hat{F}(\X^*)-\hat{F}(\Xt) - \lambda_t \|\Xt - \proj_\C(\Xt)\|_F^2 \nonumber \\
    & = \hat{F}(\X^*)-\hat{F}_t(\Xt) - \frac{\lambda_t}{2} \dist^2(\Xt,\C),
\end{align}
where we used non-expansive of the projection in the second inequality.
Combining \eqref{algo1-convex: eq1} and  \eqref{algo1-convex: eq2}, and subtracting $\hat{F} ({\X^*})$ from both sides,
we get
\begin{align} \label{appendix: dummy relation before subtraction}
	\hat{F}_{t}(\X^{t+1}) - \hat{F} ({\X^*})
	&\leq (1-\eta_t) \big( \hat{F}_{t} ({\Xt}) - \hat{F} ({\X^*}) \big) - \frac{\eta_t \lambda_t}{2}  \dist^2(\Xt,\C) +
	\frac{\Lhat \eta_t^2}{2} n D^2.
\end{align}
To arrive at telescopic terms, we need to re-write $\hat{F}_t(\Xt)$ based on $\hat{F}_{t-1}(\Xt)$ on the right-hand side. Note that 
\begin{align} 
  \hat{F}_t  ({\Xt}) 
  & = \hat{F}_{t-1}({\Xt}) + \frac{1}{2}(\lambda_t - \lambda_{t-1}) \dist^2(\Xt,\C).
  \label{appendix: Fhatt relation with t-1}
\end{align}
Substituting \eqref{appendix: Fhatt relation with t-1} into \eqref{appendix: dummy relation before subtraction}, we obtain
\begin{align}
\label{appendix:eq: dummy1}
    \hat{F}_{t}(\X^{t+1}) - \hat{F} ({\X^*})
    & \leq (1-\eta_t) \Big(\hat{F}_{t-1}(\mathbf{\Xt}) - \hat{F}({\X^*})\Big) + \frac{1}{2} \Big( (1-\eta_t) (\lambda_t - \lambda_{t-1}) - \eta_t\lambda_t \Big) \dist^2(\Xt,\C) + \frac{\Lhat \eta_t^2}{2} n D^2 \nonumber \\
    &\leq (1-\eta_t) \Big( \hat{F}_{t-1}(\mathbf{\Xt}) - \hat{F} ({\X^*})\Big) +	\frac{\Lhat \eta_t^2}{2} n D^2,
\end{align}
where the second line holds because, for all $t \geq 1$, we have
\begin{align}
    (1-\eta_t) (\lambda_t - \lambda_{t-1}) - \eta_t\lambda_t
    & = \lambda_0 \left( \frac{t-1}{t+1} (\sqrt{t+1} - \sqrt{t}) - \frac{2}{\sqrt{t+1}} \right) \nonumber \\
    & = \frac{\lambda_0}{\sqrt{t+1}} \left( \frac{t-1}{t+1} (t+1 - \sqrt{t+1}\sqrt{t}) - 2 \right) \nonumber \\
    & \leq \frac{\lambda_0}{\sqrt{t+1}} \left( \frac{t-1}{t+1} - 2 \right) 
    \leq - \frac{\lambda_0}{\sqrt{t+1}} 
    \leq 0.
\end{align}
By recursively applying the inequality in \eqref{appendix:eq: dummy1}, we get
\begin{align}
	\hat{F}_{t}(\X^{t+1}) - \hat{F} ({\X^*}) 
 & \leq \Big(\hat{F}_{0}(\X^1) - \hat{F} ({\X^*}) \Big) \prod_{i=1}^t (1-\eta_i) + \frac{nD^2}{2} \sum_{i=1}^t \hat{L}_i \eta_i^2 \prod_{j=i+1}^t (1-\eta_j) 
 \nonumber \\
 & = \frac{nD^2}{2} \sum_{i=1}^t \hat{L}_i \eta_i^2 \prod_{j=i+1}^t (1-\eta_j),
 \label{appendix:eq: dummy-new}
\end{align}
where the equality holds simply from $\eta_1 = 1$. Substituting $\eta_t = 2/(t+1)$, $\lambda_t = \lambda_0 \sqrt{t+1}$, and $\hat{L}_t:=\left(\frac{L}{n}+\lambda_t\right)$ into \eqref{appendix:eq: dummy-new}, it is easy to verify that 
\begin{align} \label{appendix: results theorem1 Fhatt}
	\hat{F}_{t}(\X^{t+1}) - \hat{F} ({\X^*}) \leq 2 nD^2\left(\frac{L/n}{t+1}+\frac{\lambda_0}{\sqrt{t+1}} \right).
\end{align}
This gives the convergence rate for the surrogate function. We want to show the convergence in terms of the original function $\hat{F}.$ To this end, we first define the Lagrangian for problem (3):
\begin{equation}
    \mathcal{L} (\X,\mathbf{R},\Y) := \hat{F}(\X) + \ip{\Y}{\X - \mathbf{R}} \quad \text{with} \quad \X \in \dom^n, ~ \mathbf{R} \in \C,
\end{equation}
where $\Y$ is the dual variable. Denote the solution to the dual problem by $\Y^*$. From the Lagrange saddle point theory, we know that the following bound holds $\forall ~ \X \in \dom^n$ and $\forall ~  \mathbf{R} \in \mathcal{C}$:
\begin{align} 
\label{appendix:eq: dummy2}
   \hat{F} ({\X^*}) \leq \mathcal{L} (\X,\mathbf{R},\Y^\star)=\hat{F} ({\X})+\ip{\Y^\star}{\X-\mathbf{R}} \leq
    \hat{F} ({\X})+\norm{\Y^\star}_F ~\norm{\X-\mathbf{R}}_F.
\end{align}
Given that $\X^{t+1} \in \dom^n$, we can write 
\begin{equation} \label{appendix: lowerbound for Fhat}
 \hat{F} ({\X^{t+1}})-\hat{F} ({\X^*}) \geq 
 -~ \min_{\mathbf{R} \in \mathcal{C}} \norm{\Y^\star}_F~\norm{\X^{t+1}-\mathbf{R}}_F = -\norm{\Y^\star}_F ~ \dist(\X^{t+1},\mathcal{\mathcal{C}}).
\end{equation}
Now, we formulate \eqref{appendix: results theorem1 Fhatt},
\begin{equation} \label{appendix: upperrbound for Fhat}
\underbrace{\hat{F} ({\X^{t+1}})-\hat{F} ({\X^*}) + \frac{\lambda_t}{2} \dist(\X^{t+1},\C)^2}_{=	\hat{F}_{t}(\X^{t+1})  -
	\hat{F} ({\X^*})} \leq 2 nD^2\left(\frac{L/n}{t+1}+\frac{\lambda_0}{\sqrt{t+1}} \right).
\end{equation}
Combining this with \eqref{appendix: lowerbound for Fhat}, we get
\begin{align}
-\norm{\Y^\star} ~ \dist(\X^{t+1},\mathcal{\mathcal{C}})+ \frac{\lambda_t}{2} \dist^2(\X^{t+1},\C) &\leq 2 nD^2\left(\frac{L/n}{t+1}+\frac{\lambda_0}{\sqrt{t+1}} \right).
\end{align}
Solving the quadratic inequality in terms of $\dist(\X^{t+1},\C)$, we obtain
\begin{align}\label{appendix: upperrbound for distance}
\dist(\X^{t+1},\C)&\leq 
\frac{1}{\lambda_t}
\left ( \norm{\Y^\star}_F
+\sqrt{\norm{\Y^\star}_F^2 +4nD^2\lambda_t
\left(\frac{L/n}{t+1}+\frac{\lambda_0}{\sqrt{t+1}} \right)
}
\right)
\nonumber\\
& \leq \frac{2}{\lambda_0 \sqrt{t+1}}
\left (\norm{\Y^\star}_F
+D \sqrt{\lambda_0
\left(\frac{L}{\sqrt{t+1}}+n\lambda_0 \right)
}
\right) \nonumber\\
& \leq \frac{2}{\lambda_0 \sqrt{t+1}}
\left (\norm{\Y^\star}_F
+D \sqrt{\lambda_0
(L+n\lambda_0)
}
\right).
\end{align}
Using convexity of $\hat{F}$, we can write
\begin{align}\label{appendix: dummy equations no1}
\hat{F}(\Xbt)-\hat{F}(\X^*) &\leq 
\hat{F}(\Xt)-\hat{F}(\X^*)+\ip{\nabla \hat{F}(\Xbt)}{\Xbt-\Xt} 
\nonumber \\
&\leq
\hat{F}(\Xt)-\hat{F}(\X^*) +\norm{\nabla \hat{F}(\Xbt)}_F\cdot \norm{\Xbt-\Xt}_F 
\nonumber \\
&\leq 2 nD^2\left(\frac{L/n}{t}+\frac{\lambda_0}{\sqrt{t}} \right)+G 
\frac{2}{\lambda_0 \sqrt{t}}
\left(  \norm{\Y^\star}
+ D \sqrt{\lambda_0(L/\sqrt{t} +n \lambda_0)}  \right)
\quad (\text{see \Cref{lem:background:def-G}})
,
\end{align}
where in the last line, we used bounds \eqref{appendix: upperrbound for Fhat} and \eqref{appendix: upperrbound for distance}. 
Finally, we note that this provides a bound for $F(\xb^t)-F(\x^*)$, since we have the following equality by definition of $\hat{F}$: 
\begin{align}\label{appendix: dummy equations no2}
\hat{F}(\Xbt)-\hat{F}(\X^*) 
= \frac{1}{n} \sum_{i=1}^n f_i(\Xbt \mathbf{e}_i)-\frac{1}{n} \sum_{i=1}^n f_i(\X^* \mathbf{e}_i)
= \frac{1}{n} \sum_{i=1}^n f_i(\xb^t)-\frac{1}{n} \sum_{i=1}^n f_i(\x^*)
=F(\xb^t)-F(\x^*).
\end{align}
This completes the proof.

{\hfill $\square$}

\subsection{Extension of Theorem~1 for Partial Client Participation}
\label{sec:partial-participation-convex}

We assume that each client participates in each round with probability $\texttt{p}$. 
Note that 
\begin{align}
\x_i^{t+1}=
\begin{cases}
 ~(1-\eta_t) \x_i^t + \eta_t \mathbf{s}_i^t, & \text{with probability } \texttt{p}
\\ 
~\x_i^{t} & \text{with probability } 1-\texttt{p}.
\end{cases}
\end{align}
Therefore, $\mathbb{E}_t [\x_i^{t+1}-\x_i^{t}]=0\times (1-\texttt{p})+ \eta_t (\mathbf{s}_i^t-\x_i^t)\times  \texttt{p}$ and $\mathbb{E}_t [\norm{\x_i^{t+1}-\x_i^{t}}^2]=0\times (1-\texttt{p})+ \eta_t^2 \norm{\mathbf{s}_i^t-\x_i^t}^2  \times \texttt{p}$. 
We start by using smoothness, to obtain
\begin{align}
	\hat{F}_{t}(\X^{t+1}) - \hat{F} ({\X^*})
	& \leq  \hat{F}_{t} ({\Xt}) - \hat{F} ({\X^*}) + 
	\ip{\nabla \hat{F}_{t} ({\Xt})}{\X^{t+1}-\Xt }
	+\frac{\Lhat}{2} \norm{\X^{t+1}-\Xt}_F^2. 
\end{align}
Taking conditional expectation given the history
prior to round $t$, we get 
\begin{align}
\mathbb{E}_t \Big[ \hat{F}_{t}(\X^{t+1})-F^* \Big]
	&\leq \mathbb{E}_t \Big[\hat{F}_{t} ({\Xt})-F^*\Big]
 + 
	\ip{\nabla \hat{F}_{t} ({\Xt})}{\mathbb{E}_t \Big[ \X^{t+1}-\Xt\Big]}
	+\frac{\Lhat}{2} \mathbb{E}_t \Big[ \| \X^{t+1}-\Xt \|_F^2 \Big]
 \nonumber \\
 &\leq  \mathbb{E}_t \Big[\hat{F}_{t} ({\Xt})-F^*\Big]
 + \eta_t \texttt{p}
	\ip{\nabla \hat{F}_{t} ({\Xt})}{\mathbf{S}^t-\Xt }
	+\frac{\Lhat \eta_t^2}{2} n \texttt{p} D^2
 \nonumber \\
	&\leq  \mathbb{E}_t \Big[\hat{F}_{t} ({\Xt})-F^*\Big] + \eta_t \texttt{p}
	\Big(
	\hat{F} ({\X^*})-\hat{F}_t  ({\Xt})-\frac{\lambda_t}{2}  \dist(\Xt,\C)^2 \Big)+
	\frac{\Lhat \eta_t^2}{2} n \texttt{p} D^2
 \quad (\text{see \eqref {appendix: dummy relation before subtraction}})
 \nonumber \\
 &\leq 
\mathbb{E}_t \Big[ \hat{F}_{t-1}  ({\Xt})+\frac{\lambda_{t}}{2} \big(
   \frac{\lambda_{t}}{\lambda_{t-1}}-1 
  \big) ~ \dist^2(\Xt,\C)-F^*\Big] 
  \nonumber \\
  &+\eta_t \texttt{p}
	\Big(
	\hat{F} ({\X^*})-\hat{F}_t  ({\Xt})-\frac{\lambda_t}{2}  \dist(\Xt,\C)^2 \Big)+
	\frac{\Lhat \eta_t^2}{2} n \texttt{p} D^2
 \quad(\text{see \eqref{appendix: Fhatt relation with t-1}})
 \nonumber \\
 &\leq (1-\eta_t \texttt{p}) \mathbb{E}_t \Big[ \hat{F}_{t-1}  ({\Xt})-F^*\Big]
 +\zeta_t ~ \dist^2(\Xt,\C)
 +	\frac{\Lhat \eta_t^2}{2} n \texttt{p} D^2
  \nonumber \\
 &\leq (1-\eta_t \texttt{p}) \mathbb{E}_t \Big[ \hat{F}_{t-1}  ({\Xt})-F^*\Big]
 +	\frac{\Lhat \eta_t^2}{2} n \texttt{p} D^2 
 \quad (\zeta_t < 0),
\end{align}
where in the second line we used the fact that the problem is block separable. By recursively applying this inequality, we get
\begin{align}
\mathbb{E}_t \Big[ \hat{F}_{t}(\X^{t+1})-F^* \Big]
	&\leq \frac{2  nD^2}{\texttt{p}}\left(\frac{L/n}{t+2/\texttt{p}}+\frac{\lambda_0}{\sqrt{t+2/ \texttt{p}}} \right)=\mathcal{O}\Big( \frac{1}{\sqrt{t+2/\texttt{p}}} \Big).
\end{align}
Following the same steps in  \eqref{appendix: dummy equations no1} and \eqref{appendix: dummy equations no2}, we can write
\begin{align}
F(\xb^t)-F^* &\leq \frac{2D^2L}{\texttt{p} ~(t+2/\texttt{p} -1)}+\frac{1}{\sqrt{t+2/\texttt{p}-1}} \left(\frac{2 nD^2 \lambda_0}{\sqrt{\texttt{p}}} +\frac{2G}{\sqrt{n}\sqrt{\texttt{p}} \lambda_0} \left(  \norm{\Y^\star}+ D \sqrt{\lambda_0(L +n \lambda_0)}  \right)\right)
\nonumber \\
&= \frac{2D^2L}{\texttt{p} t+2-\texttt{p}}+\frac{1}{\sqrt{\texttt{p}t+2-\texttt{p}}} \left(2 nD^2 \lambda_0 +\frac{2G}{\sqrt{n} \lambda_0} \left(  \norm{\Y^\star}+ D \sqrt{\lambda_0(L +n \lambda_0)}  \right)\right).
\end{align}
This completes the proof.
{\hfill $\square$}

\section{Convergence Analysis of \textsc{FedFW} in the Non-Convex Case} \label{sec: appendix: theorem 2}
In this section, we provide the proofs for  %
Theorem~2 and its generalization to the sampled clients scenario.
\subsection{Proof of Theorem~2} \label{sec: appendix: theorem 2 subsection 1}
Define the surrogate gap function:
\begin{equation}
    \sgap(\X) := \max_{\mathbf{U} \in \dom^n}\langle \nabla \hat{F}_T ({\X}), \X - \mathbf{U}\rangle. 
\end{equation}
The proof of this theorem involves three parts. 
In part {\bf (a)}, we show that the surrogate gap function converges as  
\begin{align}
\sgap(\X^{\tilde{t}}) := \min_{1\leq t \leq T} ~ \sgap(\Xt) &\leq \mathcal{O}\left(\frac{1}{T^{\frac{1}{3}}}\right).
\label{eq: appendix: generalization to non-connvex case part1}
\end{align}
Note that this does not immediately imply that the original gap function converges to zero. To address this, in part {\bf (b)}, we demonstrate that $\Xt$ is converging toward the consensus set:
\begin{align}
\dist(\X^{\tilde{t}} , \mathcal{C}) &\leq \mathcal{O}\left(\frac{1}{T^{\frac{1}{3}}}\right).
\label{eq: appendix: generalization to non-connvex case part2}
\end{align}
Finally, in part {\bf (c)}, we combine these two results to demonstrate that the original gap function converges at the same rate: 
\begin{align}
    \min_{1\leq t \leq T} ~ \gap(\xb^t) \leq \mathcal{O}(T^{-1/3}).
\end{align}

\vspace{0.5em}
\noindent
\textbf{(a)}~
We consider the parameter choices 
$\eta_t=\eta=T^{-2/3}$ and $\lambda_t=\lambda=\lambda_0 T^{1/3}$ for all $t = 1, \ldots, T$. 
Similarly, the surrogate function $\hat{F}_1=\ldots=\hat{F}_T$. 
Moreover, by \Cref{lem:sec: appendix: SmoothnessOfhhat}, we know that $\hat{F}_t$ is smooth with parameter $\hat{L}_t = \hat{L} = \frac{L}{n} + \lambda$. Then, we have
\vspace{-0.5em}
	\begin{align}
	  \label{algo1-nonconvex: eq1}
	\hat{F}_T(\X^{t+1})
	&\leq \hat{F}_T ({\Xt}) + 
	\ip{\nabla \hat{F}_T ({\Xt})}{\X^{t+1}-\Xt}
	+\frac{\hat{L}}{2} \| \X^{t+1}-\Xt \|_F^2
	\nonumber \\
	&= \hat{F}_T ({\Xt}) + \eta
	\ip{\nabla \hat{F}_T ({\Xt})}{\mathbf{S}^t-\Xt }
	+\frac{\hat{L} \eta^2}{2} \| \mathbf{S}^t-\Xt \|_F^2
	\nonumber \\
	&\leq  \hat{F}_T ({\Xt}) + \eta
	\ip{\nabla \hat{F}_T ({\Xt})}{\mathbf{S}^t-\Xt }
	+\frac{\hat{L} \eta^2}{2} n D^2
	\nonumber \\
	&= \hat{F}_T ({\Xt}) -\eta \, \sgap(\X^t) 
	+\frac{\hat{L}\eta^2}{2}   nD^2.
	\end{align}
\noindent
Rearranging \eqref{algo1-nonconvex: eq1}, taking the sum of both sides from $t=1$ to $t=T$ and dividing by $\eta T$, we get 
\begin{align} \label{eq: algo1-nonconvex 2}
\frac{1}{T} \sum_{t=1}^{T} \sgap(\X^t) &\leq \frac{1}{\eta T}  \sum_{t=1}^{T} \left (\hat{F}_T ({\Xt})-\hat{F}_T ({\X^{t+1}})   \right)
	+\frac{\hat{L}\eta}{2} n D^2
	\nonumber \\
&=\frac{1}{\eta T} \left (\hat{F}_T ({\X^1})-\hat{F}_T ({\X^{T+1}})   \right)
	+\frac{\hat{L}\eta}{2}   nD^2.
\end{align}
\vspace{-0.5em}
Note that 
\begin{align} \label{eq: dummy_non_convex1}
\hat{F}_T ({\X^1}) - \hat{F}_T ({\X^{T+1}}) 
&=
\frac{1}{n} \sum_{i=1}^n f_i(\X^1 \mathbf{e}_i) + \frac{\lambda}{2} \dist^2(\X^1,\C) 
-
\frac{1}{n} \sum_{i=1}^n f_i(\X^{T+1} \mathbf{e}_i) - \frac{\lambda}{2} \dist^2(\X^{T+1},\C)
\nonumber \\
&\leq
\frac{1}{n} \sum_{i=1}^n f_i(\X^1 \mathbf{e}_i)
-
\frac{1}{n} \sum_{i=1}^n f_i(\X^{T+1} \mathbf{e}_i)
\nonumber \\
&=
\frac{1}{n} \sum_{i=1}^n f_i(\x^1)
-
\frac{1}{n} \sum_{i=1}^n f_i(\x_i^{t+1})
\nonumber \\
&\leq
\frac{1}{n} \sum_{i=1}^n f_i(\x^1)
-
\frac{1}{n} \sum_{i=1}^n \min_{\x \in \dom} f_i(\x)
:= \mathcal{E},
\end{align}
where the second line holds since we choose $\X^1 \in \C$. 
Note that $\mathcal{E}$ is bounded above because the functions $f_i$ are smooth and the set $\dom$ is compact.
Substituting \eqref{eq: dummy_non_convex1} into \eqref{eq: algo1-nonconvex 2} and using $\eta=T^{-\frac{2}{3}},\lambda=\lambda_0 T^{\frac{1}{3}}, \hat{L}:=\frac{L}{n}+\lambda$, we get
\begin{align} \label{eq: nonconvex conv rate 1}
\frac{1}{T} \sum_{t=1}^{T} \sgap(\X^t) &\leq
\frac{ \mathcal{E}}{T\cdot T^{-\frac{2}{3}}}+\frac{(L/n+\lambda_0 T^{\frac{1}{3}})T^{-\frac{2}{3}}}{2}   nD^2
\nonumber \\
&=\left( \mathcal{E} +\frac{nD^2 \lambda_0 }{2}\right)  \frac{1}{T^{\frac{1}{3}}} +\frac{L D^2 }{2} \frac{1}{T^{\frac{2}{3}}}.
\end{align}
We complete the first part of the proof by noting that $\sgap(\X^{\tilde{t}}) :=\min_{1\leq t \leq T} \sgap(\X^t) \leq \frac{1}{T} \sum_{t=1}^T \sgap(\X^t)$.

\vspace{1em}
\noindent
\textbf{(b)}~
In the second part, our goal is to show that $\dist(\Xt,\C)$ is converging to $0$. 
We start with the following bound:
\begin{align}
\sgap(\Xt)&=
\max_{\mathbf{U} \in \dom^n}\langle  \nabla \hat{F}_{T} ({\Xt}) , \Xt - \mathbf{U}   \rangle
\nonumber \\
&\geq
\langle  \nabla \hat{F}_{T} (\Xt) , \Xt - \Xbt  \rangle
\nonumber \\
&=
\langle 
\nabla \hat{F} (\Xt) + \lambda (\Xt - \Xbt)
, \Xt- \Xbt \rangle
\nonumber \\
&=
\langle 
\nabla \hat{F} (\Xt) 
, \Xt- \Xbt \rangle +\lambda \norm{\Xt- \Xbt}_F^2.
\end{align}
Then, by using the Cauchy-Schwarz inequality, and the gradient bound in \Cref{lem:sec: appendix: SmoothnessOfhhat}, we can write
\begin{equation} \label{eq: nonconvex distance 2}
\sgap(\Xt) 
\geq -G \, \dist(\Xt,\C) + \lambda \, \dist^2(\Xt,\C),
\end{equation}
where we also used the fact that $\dist(\Xt,\C) = \norm{\Xt- \Xbt}_F$. 
By solving this quadratic inequality, we get
\begin{equation}
    \dist(\Xt,\C) \leq \frac{G}{\lambda} + \sqrt{\frac{\sgap(\Xt)}{\lambda}}.
\end{equation}
We denote $\tilde{t} \in \mathrm{argmin}_{1\leq t \leq T} ~ \sgap(\Xt)$. Then, for iteration $\tilde{t}$, by using \eqref{eq: nonconvex conv rate 1}, we can write
\begin{equation}
    \dist(\Xtt,\C) \leq \frac{G}{\lambda} + \sqrt{\frac{1}{\lambda} \left( \Big( \mathcal{E}+\frac{nD^2 \lambda_0 }{2}\Big)  \frac{1}{T^{\frac{1}{3}}} +\frac{L D^2 }{2} \frac{1}{T^{\frac{2}{3}}}\right)}.
\end{equation}
Plugging $\lambda=\lambda_0 T^{\frac{1}{3}}$, we get
\begin{align}
    \dist(\Xtt,\C) 
    \leq \frac{G}{\lambda_0 T^{\frac{1}{3}}} + \sqrt{ \left( \Big( \frac{\mathcal{E}}{\lambda_0}+\frac{nD^2}{2}\Big)  \frac{1}{T^{\frac{2}{3}}} +\frac{L D^2 }{2 \lambda_0} \frac{1}{T}\right)}. \label{eq:nonconvex-bound-dist}
\end{align}
This completes the second part of the proof. 

\vspace{1em}
\noindent
\textbf{(c)} 
Finally, we combine the results from the previous two parts to prove Theorem~2. 
We start with
\begin{align}
\sgap(\Xtt) 
&=
\max_{\mathbf{U} \in \dom^n}\langle  \nabla \hat{F}_{T} ({\Xtt}) , \Xtt- \mathbf{U}   \rangle
\nonumber \\
&=\max_{\mathbf{U} \in \dom^n}  \big\{ \ip{\nabla \hat{F}(\Xtt)}{ \Xtt-\mathbf{U} 
}
+\lambda 
\ip{\Xtt-\Xbtt}{ \Xtt-\mathbf{U} }
\big\}
\nonumber \\ 
&=
\max_{\mathbf{U} \in \dom^n}  \big\{ \ip{\nabla \hat{F}(\Xtt)}{ \Xtt-\mathbf{U} 
}
+\lambda 
\ip{\X^{\tilde{t}}-\Xb^{\tilde{t}}}{ \Xb^{\tilde{t}}-\mathbf{U} }
\big\}
+
\lambda \norm{\Xtt-\Xbtt}_F^2
\nonumber \\ 
& \geq
\max_{\mathbf{U} \in \dom^n}  \big\{ \ip{\nabla \hat{F}(\X^{\tilde{t}})}{ \X^{\tilde{t}}-\mathbf{U} 
}
+ \lambda \ip{\Xtt-\Xbtt}{ \Xbtt-\mathbf{U} } \big\}.
\end{align}
Note that $\Xbtt = \proj_{\C}(\Xtt)$. Hence, from the variational inequality formulation of projection, we have
\begin{equation}
    \ip{\Xtt-\Xbtt}{ \Xbtt-\mathbf{U} } \geq 0, \quad \forall \mathbf{U} \in \C.
\end{equation}
Combining the last two inequalities, we can write
\begin{align}
\sgap(\Xtt) 
& \geq \max_{\mathbf{U} \in \dom^n \cap \C} \, \ip{\nabla \hat{F}(\Xtt)}{ \Xtt-\mathbf{U} } \nonumber \\
& = \max_{\mathbf{U} \in \dom^n \cap \C} \, \big\{ \ip{\nabla \hat{F}(\Xbtt)}{ \Xbtt-\mathbf{U} } + \ip{\nabla \hat{F}(\Xbtt)}{ \Xtt-\Xbtt } + \ip{\nabla \hat{F}(\Xtt) - \nabla \hat{F}(\Xbtt)}{ \Xtt-\mathbf{U} } \big\} \nonumber \\
& \geq \max_{\mathbf{U} \in \dom^n \cap \C} \, \big\{ \ip{\nabla \hat{F}(\Xbtt)}{ \Xbtt-\mathbf{U} } - \norm{\nabla \hat{F}(\Xbtt)}_F \norm{ \Xtt-\Xbtt }_F - \norm{\nabla \hat{F}(\Xtt) - \nabla \hat{F}(\Xbtt)}_F \norm{ \Xtt-\mathbf{U} }_F \big\} \nonumber \\
& \geq \max_{\mathbf{U} \in \dom^n \cap \C} \, \big\{ \ip{\nabla \hat{F}(\Xbtt)}{ \Xbtt-\mathbf{U} } \big\} - (G + LD) \, \dist(\Xtt,\C) \label{eq:nonconvex-gap-gap-dist},
\end{align}
where the second inequality uses Cauchy-Schwarz inequality, and the last inequality is based on the smoothness of $\hat{F}$, boundedness of $\dom^n$, and the boundedness of the gradient (see \Cref{lem:background:def-G}).
Next, we note that 
\begin{align}
    \max_{\mathbf{U} \in \dom^n \cap \C} \,  \ip{\nabla \hat{F}(\Xbtt)}{ \Xbtt-\mathbf{U} } 
    & = \max_{\mathbf{u} \in \dom} \,  \ip{\frac{1}{n} \sum_{i=1}^n \nabla f_i(\xb^{\tilde{t}}) \mathbf{e}_i^\top}{ (\xb^{\tilde{t}}-\mathbf{u})\mathbf{1}^\top } 
    \nonumber \\
    & = \max_{\mathbf{u} \in \dom} \,  \ip{\frac{1}{n} \sum_{i=1}^n \nabla f_i(\xb^{\tilde{t}})}{ \xb^{\tilde{t}}-\mathbf{u}} 
    \nonumber \\
    & = \max_{\mathbf{u} \in \dom} \,  \ip{\nabla F(\xb^{\tilde{t}})}{ \xb^{\tilde{t}}-\mathbf{u}} 
    \nonumber \\
    & = \gap(\xb^{\tilde{t}}). \label{eq:nonconvex-gap-smoothgap-relation}
\end{align}
By combining \eqref{eq:nonconvex-gap-gap-dist} with \eqref{eq:nonconvex-gap-smoothgap-relation} and rearranging, we get
\begin{equation}
    \gap(\xb^{\tilde{t}}) \leq \sgap(\Xtt) + (G+LD) \, \dist(\Xtt,\C).
\end{equation}
Finally, we complete the proof by noting that $\min_{1\leq t \leq T} \gap(\xb^t) \leq \gap(\xb^{\tilde{t}}) \leq \mathcal{O}(1/T^{\frac{1}{3}})$ based on \eqref{eq: nonconvex conv rate 1} and \eqref{eq:nonconvex-bound-dist}.
{\hfill $\square$}

\subsection{Extension of Theorem~2 for Partial Client Participation}
\label{sec:partial-participation-nonconvex}

We assume that each client participates in each round with probability $\texttt{p}$. 
Note that 
\begin{align}
\x_i^{t+1}=
\begin{cases}
 (1-\eta_t) \x_i^t + \eta_t \mathbf{s}_i^t, & \text{with probability } \texttt{p}
\\ 
\x_i^{t} & \text{with probability } 1-\texttt{p}.
\end{cases}
\end{align}
Therefore, $\mathbb{E}_t [\x_i^{t+1}-\x_i^{t}]=0\times (1-\texttt{p})+
\eta_t (\mathbf{s}_i^t-\x_i^t)\times  \texttt{p}$ and $\mathbb{E}_t [\norm{\x_i^{t+1}-\x_i^{t}}^2]=0\times (1-\texttt{p})+
\eta_t^2 \norm{\mathbf{s}_i^t-\x_i^t}^2  \times \texttt{p}$. 
Similarly to the proof of Theorem~2, we start with \eqref{algo1-nonconvex: eq1}. Taking conditional expectation given the history
prior to round $t$, we get     
    \begin{align}
	\mathbb{E}_t \big[ \hat{F}_t(\X^{t+1}) - \hat{F}_t ({\Xt}) \big]
	&\leq \ip{\nabla \hat{F}_t ({\Xt})}{\mathbb{E}_t \big[ \X^{t+1}-\Xt \big]}
	+\frac{\hat{L}}{2} \mathbb{E}_t \big[ \| \X^{t+1}-\Xt \|_F^2]
	\nonumber \\
	&\leq   \eta_t \texttt{p}
	\ip{\nabla \hat{F}_t ({\Xt})}{\mathbf{S}^t-\Xt }
	+\frac{\hat{L} \eta_t^2}{2} n \texttt{p} D^2
	\nonumber \\
	&=  -\eta_t \texttt{p} \mathbb{E}_t \big[ \sgap(\X^t) \big] 
	+\frac{\hat{L}\eta^2_t}{2}   n \texttt{p}D^2.
	\end{align}
Now, by following the steps in \Cref{sec: appendix: theorem 2 subsection 1} with $\eta={(\texttt{p} T+1)}^{-\frac{2}{3}},\lambda=\lambda_0 {(\texttt{p} T+1)}^{\frac{1}{3}}$, we can write
\begin{align} \label{eq: nonconvex conv rate 1 subsection2}
\frac{1}{T} \sum_{t=1}^{T} \ \mathbb{E}_t \big[ \sgap(\X^t) \big] 
=\mathbb{E} \big[ \sgap(\X^t) \big] 
&\leq \big( \mathcal{E}+\frac{nD^2 \lambda_0 }{2}\big)  \frac{1}{(\texttt{p} T)^{\frac{1}{3}}} +\frac{L D^2 }{2} \frac{1}{(\texttt{p} T)^{\frac{2}{3}}},
\end{align}
where $\mathbb{E}[\cdot]$ is the expectation with respect to the all randomness. 
The rest of the proof follows similarly.
{\hfill $\square$}

\section{Convergence Analysis of \textsc{FedFW} in the Stochastic Case} \label{sec:appendix-fw-sto}

This section presents the convergence guarantee of \textsc{FedFW-sto} for the convex case, as shown in Algorithm~1. 
Here, we assume access to stochastic gradients that satisfy the bounded variance assumption.

\textsc{FedFW-sto} can be viewed as an instance of the Stochastic Homotopy Conditional Gradient Method (SHCGM) as described in \cite{locatello2019stochastic}, applied to the stochastic FL template presented in Section 4.1 of our main paper. 
Accordingly, we have the following lemma, directly from Theorem~2 in \cite{locatello2019stochastic}.

\begin{lemma}\label{sec:appendix:sto-bounds-lem}
Denote by $\hat{F}(\mathbf{X}) := \frac{1}{n} \sum_{i=1}^n \mathbb{E}_{\omega_i} \big[ 
f_i(\mathbf{X} \mathbf{e}_i, \omega_i) \big]$. 
Suppose each $f_i(\x) := \mathbb{E}_{\omega_i} \big[f_i(\x, \omega_i) \big]$ is convex and $L$-smooth. 
Further assume that the bounded variance assumption (see Assumption~1 in the main paper) holds. 
Then, the sequence $\Xt$ generated by \textsc{FedFW-sto} satisfies 
\begin{align}
\label{sec:appendix:sto-bounds:eq1}
&\mathbb{E} \big[ \hat{F}(\Xt) \big] -\hat{F} ({\X^*})
 \leq 9^{\frac{1}{3}} \frac{C}{(t+7)^\frac{1}{3}}
\\
\label{sec:appendix:sto-bounds:eq2}
&\mathbb{E} \big[\dist^2(\Xt,\C) \big]  \leq \frac{2 \norm{\Y^\star}_F}{\lambda_0 \sqrt{t+7}}
+
\frac{2 ~ \sqrt{2 \cdot 9^{\frac{1}{3}} C/\lambda_0 }}{(t+7)^{\frac{5}{12}}},
\end{align}
 where $C = \frac{81}{2} nD^2(L/n+\lambda_0)+9D\sqrt{Q}, \ Q = \max \big \{ \nsq{\nabla \hat{F} ({\X^1})-\D^1}_F 7^{2/3}, 16 n \sigma^2+ 81 L^2 D^2 / n \big\}$, $\D^1$ is the matrix of concatenated $\bd_i^1$, and $\Y^\star$ is a solution to the Lagrange dual problem.
\end{lemma}
\Cref{sec:appendix:sto-bounds-lem} is a direct application of Theorem~2 in \cite{locatello2019stochastic}. 
We omit the proof, and refer to the original source for more details. It is worth noting that strong duality holds, assuming that the relative interior of $\dom$ is non-empty, since the problem is convex and Slater's condition is satisfied.

\subsection{Proof of Theorem~3}
Using convexity of $\hat{F}$, we can write
\begin{align}
\mathbb{E}  \big[ \hat{F}(\Xbt) \big] -\hat{F}(\X^*) &\leq 
\mathbb{E}  \big[ \hat{F}(\Xt) \big]-\hat{F}(\X^*)+
\mathbb{E}  \big[  \ip{\nabla \hat{F}(\Xbt)}{\Xbt-\Xt} \big]
\nonumber \\
&\leq
\mathbb{E}  \big[ \hat{F}(\Xt) \big] -\hat{F}(\X^*) +
\mathbb{E}  \Big[ 
\underbrace{\norm{\nabla \hat{F}(\Xbt)}_F}_{\leq G} \cdot \underbrace{\norm{\Xbt-\Xt}_F }_{=\dist^2(\Xt,\C)} \Big]
\nonumber \\
&\leq 9^{\frac{1}{3}} \frac{C}{(t+7)^\frac{1}{3}}
+G \Big(
 \frac{2 \norm{\Y^\star}_F}{\lambda_0 \sqrt{t+7}}
+
\frac{2 ~ \sqrt{2 \cdot 9^{\frac{1}{3}} C/\lambda_0 }}{(t+7)^{\frac{5}{12}}} \Big),
\end{align}
where we used \Cref{lem:background:def-G,sec:appendix:sto-bounds-lem}. 
Finally, we note that this provides a bound for $\mathbb{E}  \big[  F(\xb^t) \big]-F(\x^*)$, since we have the following equality by the definition of $\hat{F}$: 
\begin{align*}\label{appendix: dummy equations no2}
\mathbb{E}  \big[  \hat{F}(\Xbt) \big] -\hat{F}(\X^*) 
&=
\mathbb{E}  \Big[ 
\frac{1}{n} \sum_{i=1}^n \mathbb{E}_{\omega_i} \big[  f_i(\Xbt \mathbf{e}_i, \omega_i^t) \big] \Big]
-\frac{1}{n} \sum_{i=1}^n 
\mathbb{E}_{\omega_i} \big[ f_i(\X^* \mathbf{e}_i, \omega_i) \big]
\\
&= \mathbb{E} \big[  F(\xb^t) \big] -F(\x^*) 
\\ &\leq \mathcal{O}(t^{-{1}/{3}}).
\end{align*}
This completes the proof.
{\hfill $\square$}

\end{document}